\documentclass{article}
\usepackage[utf8]{inputenc}
\usepackage{textcomp}
\usepackage{amssymb} 
\usepackage{newunicodechar}
\usepackage{placeins}
\newunicodechar{−}{\ensuremath{-}}
\usepackage{amsmath}
\usepackage{enumitem}
\usepackage{amsthm}
\newtheorem{theorem}{Theorem}
\newtheorem{lemma}[theorem]{Lemma}
\newtheorem{proposition}[theorem]{Proposition}
\newtheorem{corollary}[theorem]{Corollary}

\PassOptionsToPackage{numbers,sort&compress}{natbib}
\usepackage[main, final]{neurips_2026}
\usepackage{graphicx}
\usepackage[utf8]{inputenc} 
\usepackage[T1]{fontenc}    
\usepackage{hyperref}       
\usepackage{url}            
\usepackage{booktabs}       
\usepackage{amsfonts}       
\usepackage{nicefrac}       
\usepackage{microtype}      

\usepackage[most]{tcolorbox}
\usepackage[table]{xcolor}
\definecolor{ourscolor}{RGB}{230,243,255}  
\definecolor{mydarkblue}{RGB}{0,80,160}

  \newtcolorbox{summarybox}{
    colback=ourscolor,
    colframe=mydarkblue,
    boxrule=1pt,
    arc=3pt,
    left=6pt,
    right=6pt,
    top=6pt,
    bottom=6pt,
  }

\newcommand{\eos}{EoS}

\newcommand{\vone}{\mathbf{v}_{1}}

\makeatletter
\newcommand\blfootnote[1]{%
  \begingroup
  \renewcommand\thefootnote{}%
  \renewcommand\@makefnmark{}%
  \renewcommand\@makefntext[1]{\noindent##1}%
  \footnotetext{#1}%
  \endgroup
}
\makeatother

\title{Edge of Stability Selectively Shapes Learning Across the Data Distribution}

\author{%
  Shauna Kwag$^{*,\dagger}$ \\
  MIT
  \And
  Anakha Ganesh$^{*,\dagger}$ \\
  MIT
  \And
  Tomaso Poggio \\
  MIT
  \And
  Pierfrancesco Beneventano$^{\dagger}$  \\
  MIT
}

\begin{document}

\maketitle

\blfootnote{* Core contributors}
\blfootnote{$^\dagger$ Correspondence to \texttt{\{kwags,anakhag,pierb\}@mit.edu }}
\begin{abstract}
Existing analyses of the edge of stability (EoS) treat it as a global property of optimization. We show that it is also selective: the stability constraint redistributes learning across subsets of the training distribution, amplifying progress on some groups while suppressing progress on others. Using a branching intervention that enters or exits the EoS regime from the same training state, we causally demonstrate this trade-off and identify two necessary conditions for a group to benefit. First, its aggregate gradient must align with the top Hessian eigenvector. We isolate this mechanism with a controlled perturbation that preserves distance but randomizes direction, destroying alignment and eliminating the advantage. Second, the group must sustain non-vanishing gradient magnitude over time. Under cross-entropy loss, gradient saturation decouples confidently classified groups, shifting the advantage to output-outliers, whose gradients persist. Together, these results show that EoS functions not only as a stability boundary, but as a mechanism governing the allocation of learning across the data distribution.
\end{abstract} 

\section{Introduction}

Deep neural networks exhibit strong sensitivity to optimizer and hyperparameters. Training choices such as learning rate, batch size, and optimizer affect which solution is found \citep{zhang2017understanding, keskar2017large, jastrzebski2018three}, unlike in the classical convex setting where these choices do not affect which minimum is reached. Understanding the mechanisms underlying this implicit bias is a key objective in the theory of deep learning.

One structural explanation comes from the \textit{edge of 
stability} (\eos{}) literature: under full-batch and large-batch 
gradient descent, the top Hessian eigenvalue self-stabilizes near 
the stability threshold that depends on the optimizer and 
hyperparameters \citep{xingwalk2018, jastrzebski2019relation, jastrzebski2020break-even, cohen2021gradient, cohen2023adaptive, andreyev2024edge, andreyev2026momentum, islamov2026non}.
At this threshold, the optimizer operates at the boundary of 
discrete-time stability, constraining which regions of the loss 
landscape remain accessible during training.

While the \eos{} phenomenon is well established, far less is understood about its consequences. In particular, it remains unclear whether operating near the stability threshold provides any functional benefit, or how these stability constraints shape optimization across the data distribution. Prior work has primarily characterized \eos{} through curvature and optimization trajectories in parameter space, leaving open how these dynamics influence which examples are learned during training. We ask:
\begin{center}
\textit{Which subsets of the training distribution benefit from \eos{}?\\ Which do not? What governs this allocation?}
\end{center}

To study this allocation, we define four prototype groups from the geometry of $P_{X,Y}$ that vary in input typicality, label consistency, and boundary proximity, independent of the model or loss. We find that \eos{} induces a \textit{selective} learning regime: the stability constraint distributes optimization effort unevenly, amplifying subsets whose gradients persistently align with the top Hessian eigendirection while suppressing others. This counters classical intuition: although the Descent Lemma treats the sharpest direction as a stability boundary that constrains learning, alignment with it is precisely what determines how learning is allocated across the distribution.

Our findings connect to a longstanding debate on whether low curvature improves generalization \citep{hochreiter1997flat, keskar2017large, foret2021sharpness, dinh2017sharp}. \eos{} is the regime in which sharpness is actively constrained near the stability threshold, making it a natural setting to test what low curvature actually confers. Our results suggest the answer is not global: the functional benefit depends on which subset dominates the top Hessian eigendirection and shifts with the geometric composition of the training distribution. Flatness, in this view, is a directional property determined by data geometry rather than a scalar property of the solution. More broadly, our results take a step toward connecting two largely separate lines of work: \textit{implicit regularization} in parameter space and \textit{inductive bias} over the data distribution.

Our paper makes the following contributions:
\begin{itemize}[leftmargin=1em,itemsep=0.6em,topsep=0.15em,parsep=0pt]
\item \textbf{\eos{} is selective, not global} (\S\ref{sec:experiments}). Using a branching intervention that enters or exits \eos{} from a shared training trajectory, we find that the stability constraint is not a uniform bottleneck: it selectively benefits some subsets of the training distribution while suppressing others. The trade-off qualitatively replicates across different architectures, optimizers, and datasets (Appendix \ref{app:robustness}, \ref{app:multiclass}).

\item \textbf{Selectivity is governed by alignment $\times$ persistence} (\S\ref{sec:mechanism}). The advantage at \eos{} is captured by subsets whose aggregate gradient both aligns with the top Hessian eigenvector and remains non-vanishing throughout training. Two controlled counterfactuals isolate each factor: random-direction displacement removes alignment, and cross-entropy saturation removes persistence. In both cases the \eos{} advantage disappears or is transferred to the subsets that retain the missing factor. Both factors fall out of extending the self-stabilization framework \citep{damian2023selfstabilization} from the global loss to per-subset losses (Section~\ref{sec:prelim-selfstab}; full derivation in Appendix~\ref{app:theory}).

\item \textbf{Geometry shifts the beneficiary} (\S\ref{sec:generalize}). In our MLP experiments on CIFAR-10 (the standard setting in prior \eos{} work \citep{cohen2021gradient}), the subset that benefits from \eos{} is the one furthest from the class centroids in the input space. Varying the geometric composition of the training distribution continuously shifts which subset this is. Preliminary evidence suggests the resulting generalization behavior shifts accordingly, with improved adversarial robustness when geometrically distant examples lie near the decision boundary, and improved out-of-distribution generalization when they lie far from the training distribution.
\end{itemize}

\section{Preliminaries}
\label{sec:prelim}

\subsection{Edge of stability}
Gradient descent in deep networks often operates at the \emph{edge of stability} (\eos{}), a regime in which training remains near the boundary between stable and unstable updates without diverging \citep{cohen2021gradient}. Consider full-batch gradient descent $\theta_{t+1} = \theta_t - \eta \nabla L(\theta_t)$ with Hessian $H_t = \nabla^2 L(\theta_t)$ and eigenpairs $(\lambda_i, v_i)$. A perturbation along $v_i$ is multiplied by $1 - \eta \lambda_i$ at each step, so discrete-time stability requires $\eta \lambda_i < 2$. \eos{} is the regime in which the top eigenvalue $\lambda_1$ (sharpness) saturates the bound:
\[
    \eta \lambda_1 \approx 2.
\]
The multiplier along the corresponding eigendirection $\vone$ is then near $-1$, producing sign-alternating oscillations that are linearly unstable yet remain bounded throughout training \citep{cohen2021gradient}.

\subsection{Self-stabilization}
\label{sec:prelim-selfstab}

The non-monotonic descent at \eos{} is explained by the \emph{self-stabilization} mechanism of Damian et al. \citep{damian2023selfstabilization}. Once $\lambda_1$ exceeds $2/\eta$, gradient descent develops a fast period-2 oscillation along $\vone$. Higher-order terms convert this oscillation into a slow corrective drift toward lower sharpness, so that the cycle-averaged dynamics remain near the active boundary $\{\lambda_1 \approx 2/\eta\}$. 

Under this cycle-averaged description, the \eos{} drift differs from ordinary gradient descent by an additional sharpness-reducing component,
\[
    -(\alpha/\beta)\nabla \lambda_1,
    \qquad
    \alpha = -\langle \nabla L, \nabla \lambda_1\rangle,
    \quad
    \beta = \|\nabla \lambda_1\|^2 .
\]
For a subset loss $\ell_k$, this extra drift contributes
\[
    -(\alpha/\beta)\langle \nabla \ell_k, \nabla \lambda_1\rangle
\]

We use this established framework to motivate our experimental design. The selector $\langle \nabla \ell_k, \nabla \lambda_1 \rangle$ depends on both a directional factor (alignment between $\nabla \ell_k$ and $\nabla \lambda_1$) and a magnitude factor (persistence of $\|\nabla \ell_k\|$ over training). These two components motivate complementary controlled interventions in Section~\ref{sec:mechanism}.

Empirically, we track the directional factor through alignment with the top Hessian eigenvector $\vone$, measuring how strongly each data subset couples to the unstable mode. Appendix~\ref{app:theory} derives the per-subset loss decomposition and shows that the measured $\cos^2\theta_k$ serves as an empirical proxy for the theoretical EoS selector $Q_k = \langle\nabla\ell_k, \nabla\lambda_1\rangle$ under the single-mode approximation. This approximation is supported by the observation that only $\lambda_1$ is near-critical at EoS onset (Appendix~\ref{app:single}). These considerations motivate our experiments: branching interventions test the effect of remaining at \eos{}, random-direction displacement isolates alignment, and cross-entropy saturation isolates gradient persistence.

\subsection{Prototype Taxonomy}
\label{sec:prototypes}
\paragraph{Definition.}
We partition the training distribution into four groups defined by the joint geometry of $P_{X,Y}$, independent of any trained model (Figure~\ref{fig:prototypes}). 
\emph{Inliers} are high-density points near the class centroid $\mu_c$ with correct labels. 
\emph{Boundary points} are in-distribution examples near the inter-class boundary, identified by high label ambiguity in their local neighborhood. 
\emph{Input-outliers} are geometrically atypical inputs, far from $\mu_c$ in input space, that retain correct labels. 
\emph{Output-outliers} are high-density inputs assigned an incorrect label. Representative examples from each group are shown in Figure~\ref{fig:prototype_visualization}. Inliers and boundary points are identified from the existing training data by ranking centroid distance and $k$-NN label ambiguity respectively, while input-outliers and output-outliers are synthetically constructed to isolate the effects of input-space atypicality and label inconsistency.

\begin{figure}[h]
\centering
\begin{minipage}[t]{0.48\linewidth}
    \centering
    \includegraphics[width=\linewidth]{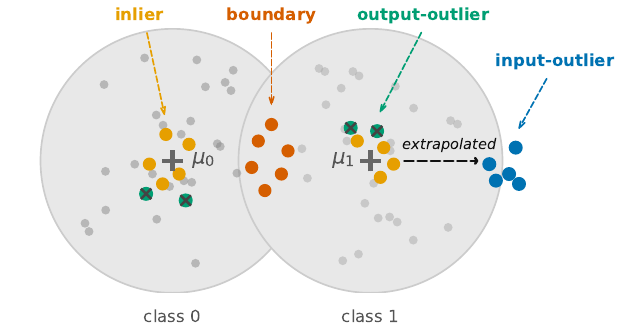}
    \caption{\textbf{Conceptual taxonomy of prototypes.} Data samples are categorized based on geometric proximity in input space relative to class-specific cluster centroids ($\mu_0$, $\mu_1$).}
    \label{fig:prototypes}
\end{minipage}
\hfill
\begin{minipage}[t]{0.48\linewidth}
    \centering
    \includegraphics[width=\linewidth]{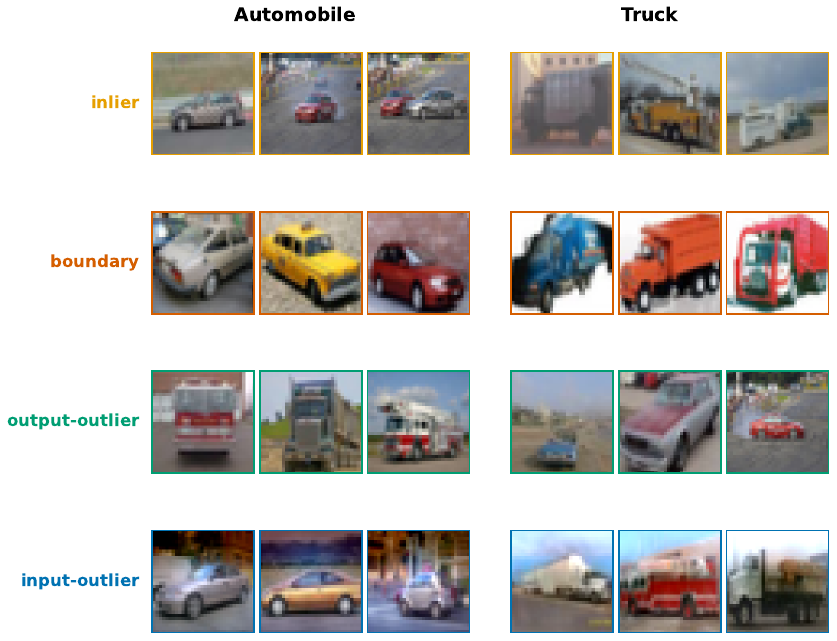}
    \caption{\textbf{Input-space visualization of prototype groups for CIFAR-10.} Three representative samples per class (automobile vs.\ truck) are shown for each group.}
    \label{fig:prototype_visualization}
\end{minipage}
\end{figure}

\paragraph{Construction.} 
We instantiate the taxonomy on a binary CIFAR-10 task (automobile vs.\ truck, $n = 10{,}000$) \citep{cifar10}. Inlier candidates are the $M = 3m$ points per class with smallest centroid distance $\|x_i - \mu_c\|$; boundary candidates are the $m$ points per class whose $k$-NN label composition ($k = 50$) is closest to uniform. From the inlier candidate pool we sample three disjoint subsets of size $m = 25$ per class: the first retains original inputs and labels (inliers); the second is assigned flipped labels $1 - c$ (output-outliers); the third is extrapolated by pushing each example away from the opposite class's centroid as $x_i \pm \alpha\, v_{\mathrm{diff}}$ (sign $+$ if $y_i = 1$, $-$ if $y_i = 0$), where $v_{\mathrm{diff}} = \mu_1 - \mu_0$ is the unnormalized centroid difference (input-outliers). We set $\alpha = 3$, chosen so input-outliers are by construction the most distant group from their class centroid; resulting pixel values may lie outside the valid input range, but we do not clip in order to preserve the displacement magnitude. Boundary points are drawn directly from the ambiguity pool. The final training set contains $n = 10{,}000$ examples, including 200 prototype-labeled points (50 per group, 25 per class) tracked throughout training.

\subsection{Metrics}
\paragraph{Per group loss $\ell_k$.}
For each prototype group $k \in \{\text{inlier, boundary, input-outlier, output-outlier}\}$ with index set $P_k$, we define $\ell_k = \frac{1}{|P_k|}\sum_{i \in P_k} \ell(f(x_i), y_i)$ as the average loss over the examples in group $k$. Tracking $\ell_k$ over training reveals the learning order across the groups and which groups are differentially affected by the stability constraint at EoS.

\paragraph{Directional coupling.}
Let $\nabla \ell_k$ denote the gradient of the loss restricted to prototype group~$k$. We measure the directional coupling of group~$k$ to the \eos{}-constrained mode by
\begin{equation}\label{eq:cos2}
    \cos^2\!\theta_k
    \;=\; \frac{(\nabla \ell_k \cdot {v}_1)^2}
               {\|\nabla \ell_k\|^2}
    \;\in\; [0,1],
\end{equation}
where $\vone$ is the top Hessian eigenvector (assumed unit norm). When $\cos^2\!\theta_k \approx 1$, the group gradient is nearly aligned with $\vone$, and when $\cos^2\!\theta_k \approx 0$, it is nearly orthogonal. This quantity measures how strongly a group’s gradient aligns with the direction constrained by \eos{} dynamics.

The link to learning comes from self-stabilization. At \eos{}, the oscillation–stabilization cycle produces net parameter movement primarily along $\vone$ \citep{damian2023selfstabilization}. Consequently, loss decreases predominantly for groups whose gradients are aligned with $\vone$, while groups with orthogonal gradients make limited progress (Figure~\ref{fig:fig_schematic}).

\begin{figure} [h]
    \centering    \includegraphics[width=1\linewidth]{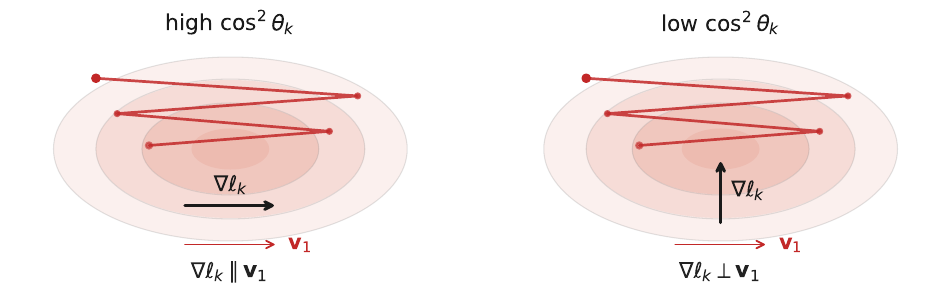}
    \caption{\textbf{Directional coupling at \eos{}.} The optimizer oscillates along $\vone$ (red zigzag). When a group's gradient $\nabla\ell_k$ aligns with $\vone$ (left), self-stabilization reduces loss for that group. When $\nabla\ell_k$ is orthogonal to $\vone$ (right), the group is decoupled from the oscillation and its loss does not benefit.}
    \label{fig:fig_schematic}
\end{figure}

\paragraph{Curvature influence.}
While $\cos^2\!\theta_k$ measures direction, it does not capture gradient magnitude. We report the squared projection \begin{equation}\label{eq:proj} (\nabla\ell_k \cdot {v}_1)^2 = \|\nabla\ell_k\|^2 \cdot \cos^2\!\theta_k, \end{equation} which quantifies a group's effective curvature influence along $\vone$. This follows from the quadratic form \[ \nabla \ell_k^\top H \nabla \ell_k = \sum_i \lambda_i (\nabla \ell_k \cdot v_i)^2, \] where the contribution of the top eigendirection is $\lambda_1 (\nabla \ell_k \cdot v_1)^2$. Since $\lambda_1$ is shared across groups at a given step, $(\nabla \ell_k \cdot v_1)^2$ ranks groups by their curvature influence in the dominant direction. The metric can decrease either through rotation away from $\vone$ or shrinking gradient magnitude. We isolate each effect in Section~\ref{sec:mechanism}.

\section{Selective Learning at the Edge of Stability}\label{sec:experiments}
\subsection{Setup and Intervention Design}\label{sec:setup}
We train a two-hidden-layer MLP (width 512, ReLU) with full-batch gradient descent on a binary CIFAR-10 task (automobile vs.\ truck, $n=10{,}000$) for 10{,}000 steps. Prototype groups are constructed as described in Section~\ref{sec:prototypes}. The default learning rate is $\eta = 0.01$ under mean square error (MSE). All quantities are plotted against effective time $t_{\mathrm{eff}} = \eta\, t$, which normalizes for step size.

To test whether the stability constraint causally affects subset-level learning, we branch each run from a shared trajectory at time $t^*$, defined as the onset of \eos{} (detected as the first time $\lambda_1$ reaches $2/\eta$). The \emph{baseline} branch continues at the original learning rate, remaining at \eos{}. The \emph{exit} branch reduces the learning rate by half ($\eta \to \eta/2$), increasing the stability threshold to $4/\eta$ and allowing for training to promptly exit the \eos{} regime. We denote $t^{**}$ as the time at which the exit branch reaches its new stability threshold $4/\eta$. The interval $[t^*, t^{**}]$ is the window over which the two branches occupy different stability regimes. Since architecture, data, and initialization are shared up to $t^*$, the branching intervention identifies the causal effect of continuing at the original learning rate versus dropping it at \eos{} onset. The intervention separates the branches into \eos{} and non-\eos{} regimes, so post-branch divergence in prototype-level loss provides evidence about the consequences of remaining at \eos{}.

All figures show medians across 5 seeds. Shading indicates interquartile range where shown. Full experimental setup details can be found in Appendix~\ref{app:details}.

\subsection{The Selective Trade-off}

Figure~\ref{fig:eos_tradeoff} shows the effect of the branching intervention under MSE training. After the exit branch leaves the \eos{} regime at $t^*$, prototype losses begin to diverge between the two runs. The divergence is group-specific: input-outlier and output-outlier loss decrease faster under the baseline ($\Delta\ell_k > 0$), while inlier and boundary loss decrease faster under the exit branch ($\Delta\ell_k < 0$). The stability constraint does not uniformly slow or accelerate learning across the data subsets; instead, it redistributes optimization, favoring some groups at the expense of others. 

The selective trade-off replicates across alternative architectures (CNN, ResNet), optimizers, and class pairs (Appendix~\ref{app:robustness}). It also transfers to binary ImageNet32, though it weakens as the number of classes grows (Appendix~\ref{app:multiclass}).

\begin{figure}[h]
    \centering    \includegraphics[width=\linewidth]{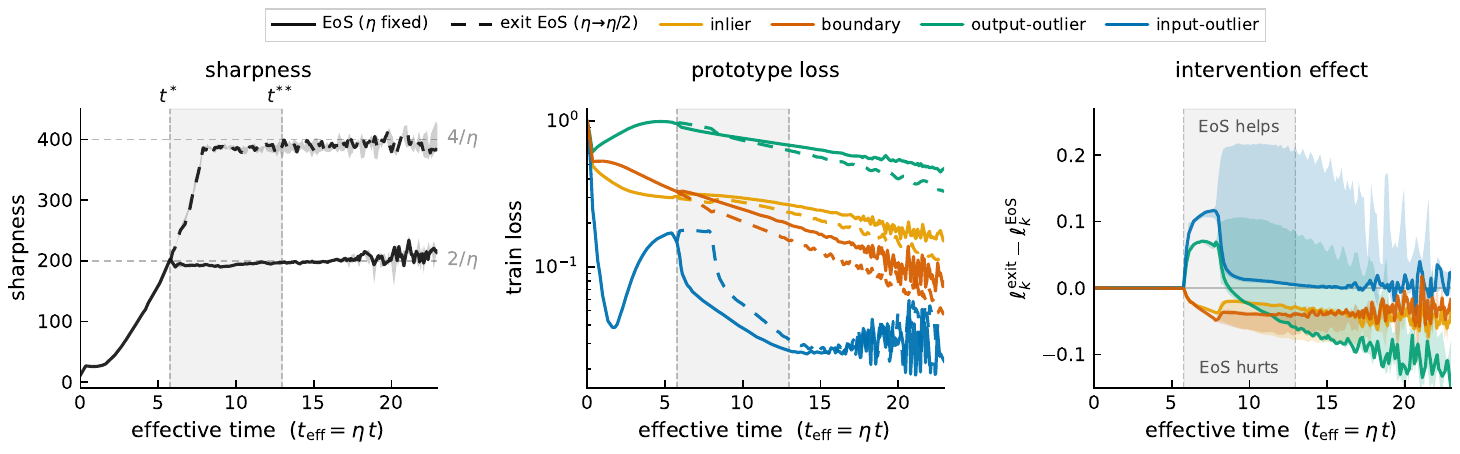}
    \caption{\textbf{EoS creates selective trade-offs across prototype groups.} Baseline run (solid) enters EoS at $t^{*}$, while the exit branch (dashed) subsequently leaves. \textbf{Left:} sharpness confirms the two branches occupy distinct stability regimes. \textbf{Middle:} prototype losses diverge post-branch, with input-outliers and output-outliers benefiting from EoS while boundary and inlier progress is suppressed. \textbf{Right:} the intervention effect $\Delta\ell_k = \ell_k^{\text{exit}} - \ell_k^{\text{EoS}}$ is shown where positive values indicate EoS achieves lower loss for that group.}    \label{fig:eos_tradeoff}
\end{figure}

\subsection{Alignment Resolves Under EoS, Persists Without It}
\label{sec:directional_alignment}

Figure~\ref{fig:alignment} shows directional alignment $\cos^2\!\theta_k$ for input-outliers and boundary points. During progressive sharpening, alignment for input-outliers increases rapidly, approaching 1 near the onset of \eos{}, while boundary alignment remains below $0.2$. As a result, input-outlier gradients dominate the top eigendirection of the Hessian, and $\vone$ aligns with them.

\begin{figure} [h]
    \centering
    \includegraphics[width=0.5\linewidth]{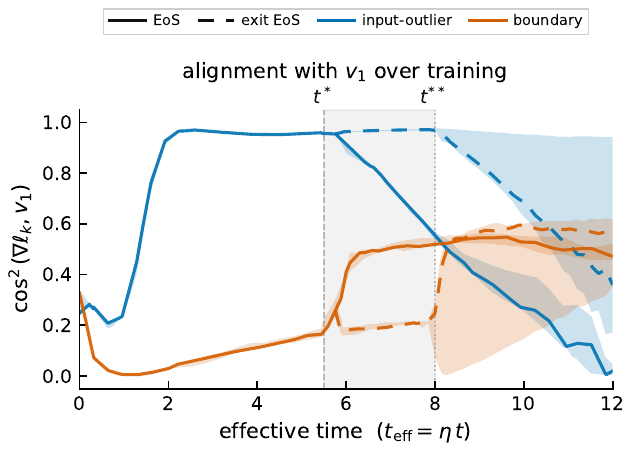}
    \caption{\textbf{Alignment dynamics under \eos{}.} Input-outlier $\cos^2\theta_k$ rises during progressive sharpening and dominates $\mathbf{v}_1$ at \eos{} onset. Under the baseline (solid), self-stabilization resolves this alignment and $\mathbf{v}_1$ rotates toward boundary points. Under the exit branch (dashed), alignment remains elevated until the new threshold is reached at $t^{**}$.}
    \label{fig:alignment}
\end{figure}

Once \eos{} is reached at $t^{*}$, input-outlier $\cos^2\!\theta_k$ declines under the baseline as self-stabilization reduces their loss and gradient magnitude. As their curvature contribution weakens, boundary alignment rises and $\vone$ rotates toward it. Under the exit branch, this transition does not occur: input-outlier alignment remains elevated through $t^{**}$, indicating that the resolution is driven specifically by \eos{} dynamics. This motivates identifying the properties that determine which group dominates $\vone$.

\section{Mechanism: Alignment and Gradient Persistence}
\label{sec:mechanism} 
Two properties are jointly necessary to capture the \eos{} advantage: directional alignment and gradient persistence. We isolate each experimentally.

\subsection{Directional alignment.} In the baseline construction, input-outliers are displaced along a shared direction $v_{\mathrm{diff}}$, yielding a directionally coherent group. To test whether this coherence drives their curvature dominance, we construct a counterfactual in which each input-outlier is displaced by the same distance but in a random direction orthogonal to $v_{\mathrm{diff}}$. The two conditions match in centroid distance, group size, and labels, differing only in directional structure (Appendix~\ref{app:coherent_random}).

Under coherent displacement along $v_{\mathrm{diff}}$, input-outlier's $\cos^2\!\theta_k$ and curvature influence $(\nabla\ell_k \cdot {v}_1)^2$ dominate (Figure~\ref{fig:coherent_incoherent}). The input-outliers benefit from \eos{} while the progress of other groups is suppressed. Under incoherent displacement, input-outlier alignment collapses, curvature influence falls by an order of magnitude, and the selective intervention effect dissolves. Geometric atypicality alone (large distance from its own class centroid) is not sufficient. The group must push the Hessian coherently along a shared direction to dominate $\vone$ and benefit at the edge of stability.

At the neuron level, our gradient analysis suggests that coherent outliers reinforce updates to shared ReLU units. In a controlled sandbox, this reinforcement accompanies the emergence of rarely active, high-magnitude neurons that selectively fit the subgroup, consistent with the neural-shattering phenomenon of Liang et al.\ \citep{liang2025neuralshattering, liang2026generalization}. Appendix~\ref{app:shattering} develops this connection and its relationship to subgroup-specific \eos{} dynamics.

\begin{figure} [h]
    \centering
    \includegraphics[width=1\linewidth]{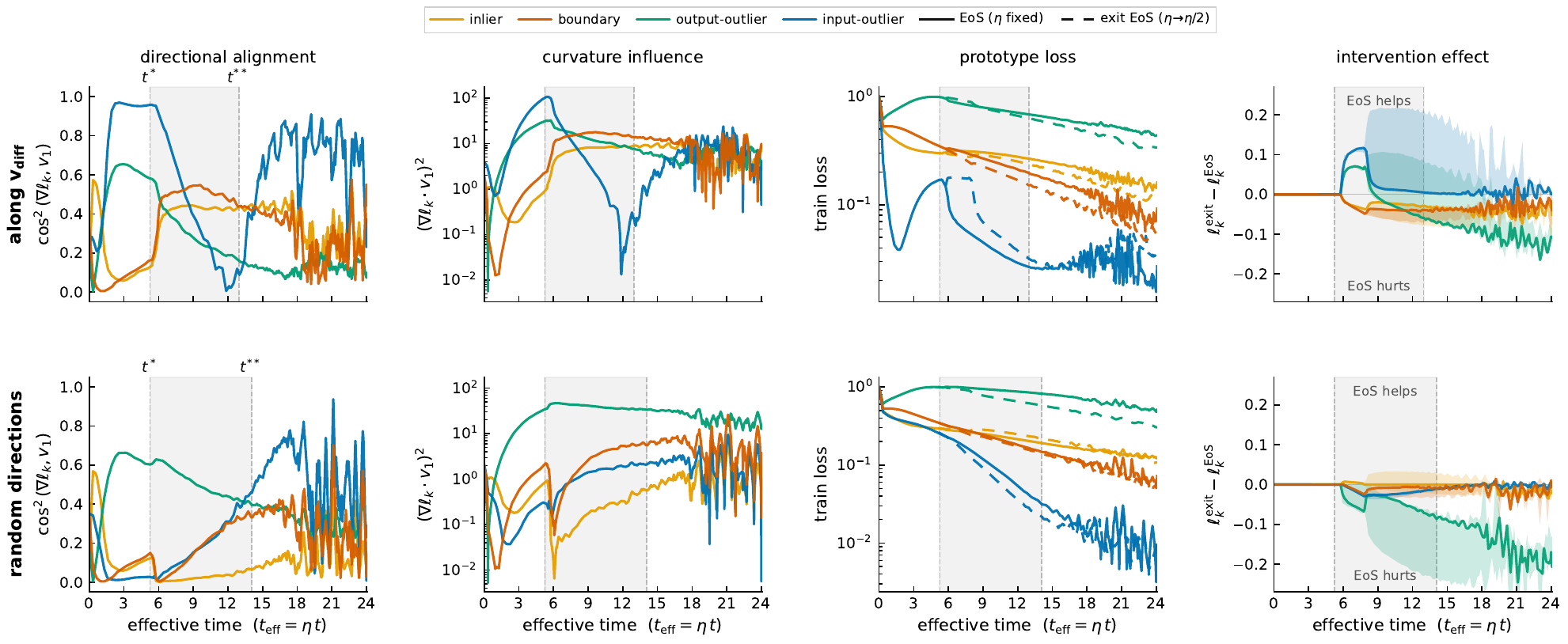}
    \caption{\textbf{Directional alignment is necessary for the selective \eos{} advantage.} Identical seeds and configurations; only the input-outlier displacement direction differs. \textbf{Top:} Coherent displacement along $v_{\mathrm{diff}}$ yields high alignment and curvature influence for input-outliers, which capture the \eos{} advantage. \textbf{Bottom:} Random orthogonal displacement at equal distance reduces alignment and curvature influence, largely eliminating the trade-off.}
    \label{fig:coherent_incoherent}
\end{figure}

\subsection{Gradient persistence.} Directional coherence determines \emph{which direction} a group pushes the Hessian, but the coupling also requires sustained gradient \emph{magnitude}. We test this by comparing MSE and CE training on identical data (Figure~\ref{fig:mse_ce}).

\begin{figure} [h]
    \centering
    \includegraphics[width=1\linewidth]{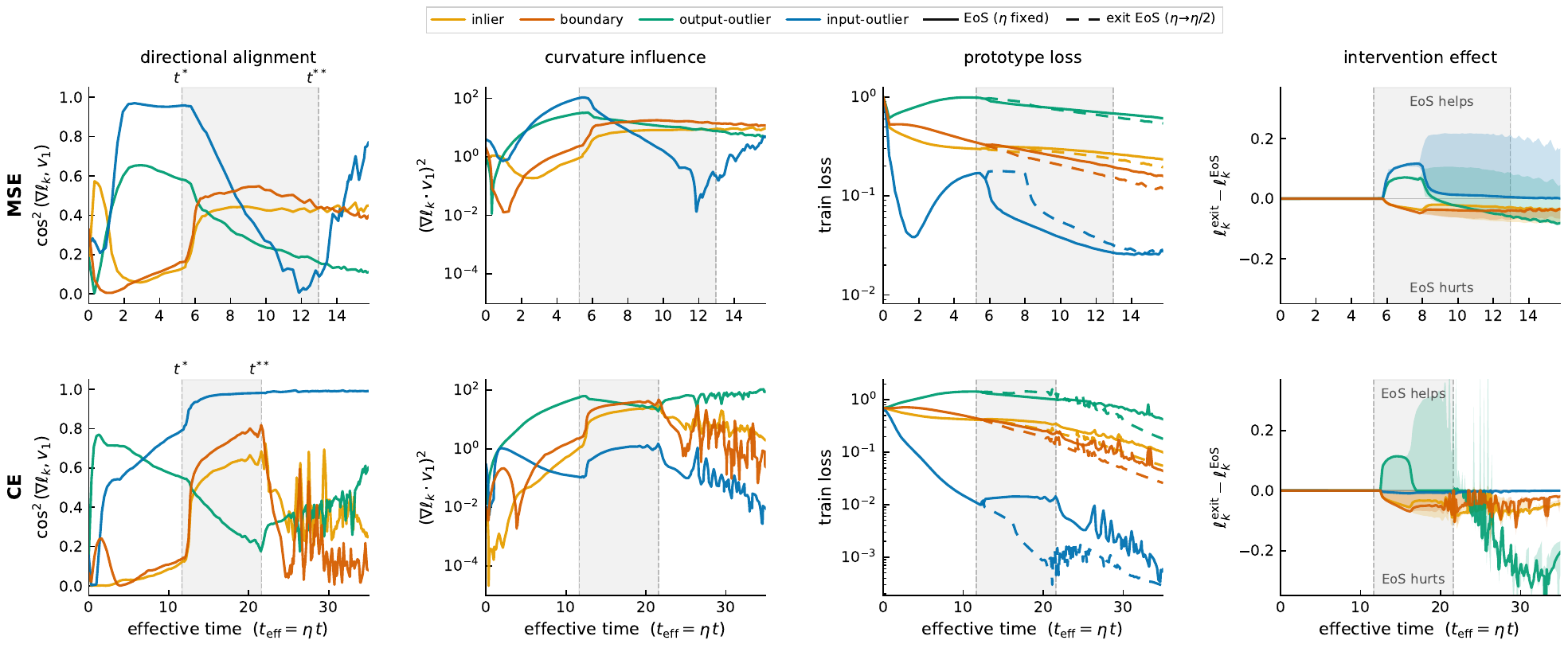}
    \caption{\textbf{Gradient persistence determines which group retains curvature influence.} Seeds and configurations are identical, and only the loss differs. \textbf{Top:} Input-outliers have elevated alignment and strong curvature influence, and they capture the \eos{} advantage. \textbf{Bottom:} Alignment for input-outliers remains high, but gradient saturation weakens their curvature influence, and the \eos{} advantage shifts to output-outliers, whose gradients remain active.}
    \label{fig:mse_ce}
\end{figure}

Under MSE as the specified loss function, per-example gradients scale with the residual and persist even for confidently classified points. Conversely, under CE, gradients vanish as points are learned and confidence grows. The experiment reveals that while the gradient for input-outliers points towards $\vone$ across both loss functions, $(\nabla\ell_k \cdot {v}_1)^2$ collapses by orders of magnitude in CE as gradient norms shrink. The functional consequence is that output-outliers, with the most elevated curvature influence, become the sole beneficiary at the edge of stability.

\subsection{Natural Outliers in the Data Distribution.} 
\label{sec:natural}
To verify that the coupling between centroid distance and $\vone$ alignment is not an artifact of synthetic prototype construction, we measure per-example directional alignment directly on the natural CIFAR-10 training distribution. For each of the $10{,}000$ training examples, we compute centroid distance $\|x_i - \mu_c\|$ and per-example alignment $\cos^2(\nabla\ell_i, v_1)$ at checkpoints throughout training. 

Figure~\ref{fig:natural} shows the relationship at two timepoints. During progressive sharpening, centroid distance and $\cos^2$ are uncorrelated (Spearman $\rho = -0.11$). At \eos{}, a positive correlation emerges ($\rho = 0.39$), where examples that are farther from their class centroid align more strongly with $\vone$. Centroid distance, a purely distributional quantity computed before training, correlates with which examples will dominate the unstable direction at \eos{}. 

\begin{figure}[h]
    \centering    
    \includegraphics[width=0.8\linewidth]{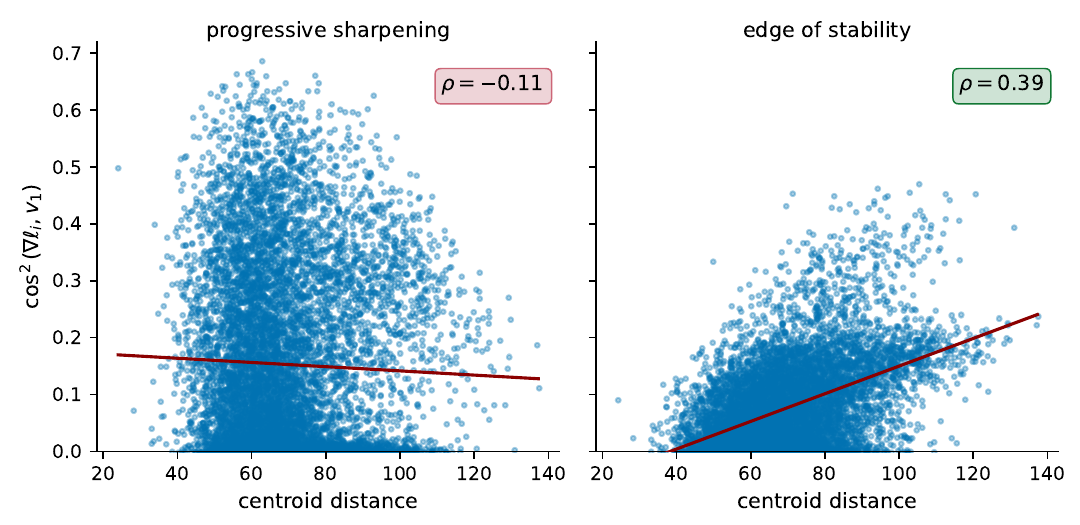}
    \caption{\textbf{Centroid distance predicts per-example 
    alignment with $\vone$ at \eos{}.} Each dot is one training 
    example; the line indicates the monotonic trend captured by Spearman $\rho$. \textbf{Left:} during progressive sharpening, 
    centroid distance and $\cos^2(\nabla\ell_i, v_1)$ are 
    uncorrelated ($\rho = -0.11$). \textbf{Right:} at \eos{}, 
    correlation emerges ($\rho = 0.39$). Full trajectory of correlation 
    is shown in Appendix~\ref{app:natural}.}
    \label{fig:natural}
\end{figure}

\section{Generalizing the Alignment Principle}
\label{sec:generalize}
Sections~\ref{sec:experiments} and~\ref{sec:mechanism} imply a clear prediction: if alignment$\times$persistence drives the effect, then changing only the data geometry, while holding the model, optimizer, and loss fixed, should change the dominant group and, in turn, the functional benefit conferred by \eos{}. 

\subsection{Performance on Harder Class Boundaries.} \label{sec:closer}
We verify this prediction on a more challenging class pair (cat vs.\ dog, $n=10{,}000$), where the relative geometry of the prototype groups shifts. With $\alpha=3$, boundary points are the most distant group from the centroid, not the input-outliers. Correspondingly, boundary points dominate $\vone$ and become the primary beneficiary of \eos{}. Increasing $\alpha$ to 10 restores input-outliers as the most distant group and transfers the advantage back to them (Appendix~\ref{app:35}).

\subsection{Does Edge-of-Stability Improve Generalization or Robustness?}
Sections~\ref{sec:experiments}--\ref{sec:mechanism} established when a subset benefits from \eos{}. We now ask whether this training-time selectivity transfers to test-time behavior and share preliminary results.

\paragraph{Adversarial robustness.} Figure~\ref{fig:alpha-adv} reports PGD adversarial accuracy ($\varepsilon = 0.03$, 10 steps, step size $\varepsilon/4$) on 50 test-set boundary points selected by $k$-NN ambiguity \citep{madry2018adv}. With $\alpha = 3$, boundary points dominate $\vone$, and the \eos{} branch maintains higher adversarial accuracy than the exit branch after $t^{**}$. The stability constraint implicitly sharpens the decision boundary in the region most relevant to robust classification. With $\alpha = 10$, where input-outliers instead dominate, the pattern reverses: the exit branch now achieves higher adversarial accuracy on boundary points, because the \eos{} branch in this configuration has been spending its optimization budget elsewhere.

\begin{figure}[h]
    \centering
    \includegraphics[width=0.8\linewidth]{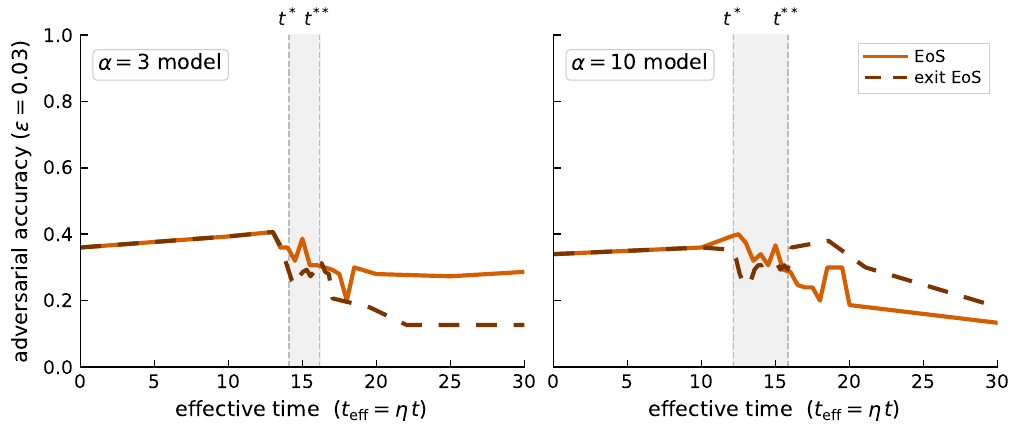}
    \caption{\textbf{EoS improves adversarial robustness only when boundary points dominate $\vone$.} \textbf{Left} ($\alpha=3$, boundary dominates $\vone$): the EoS branch (solid) outperforms the exit branch (dashed) after $t^{**}$. \textbf{Right} ($\alpha=10$, input-outlier dominates $\vone$): the pattern reverses, and the exit branch performs better. Robustness gains appear only when \eos{} prioritizes the evaluated subset. Single seed.}
    \label{fig:alpha-adv}
\end{figure}

\paragraph{Out-of-distribution generalization.} 

Figure~\ref{fig:alpha-ood} reports MSE loss on input-outliers constructed at varying $\alpha_{\text{test}}$, evaluated at the checkpoint immediately after $t^{**}$ in training. With $\alpha = 3$, the exit branch achieves lower loss at large $\alpha_{\text{test}}$, indicating no OOD advantage for input-outliers. With $\alpha = 10$, where input-outliers dominate $\vone$, the \eos{} branch achieves lower loss at large $\alpha_{\text{test}}$: optimization on input-outliers at EoS during training appears to transfer to OOD generalization along $v_{\mathrm{diff}}$.

\begin{figure}[h]
    \centering
    \includegraphics[width=0.8\linewidth]{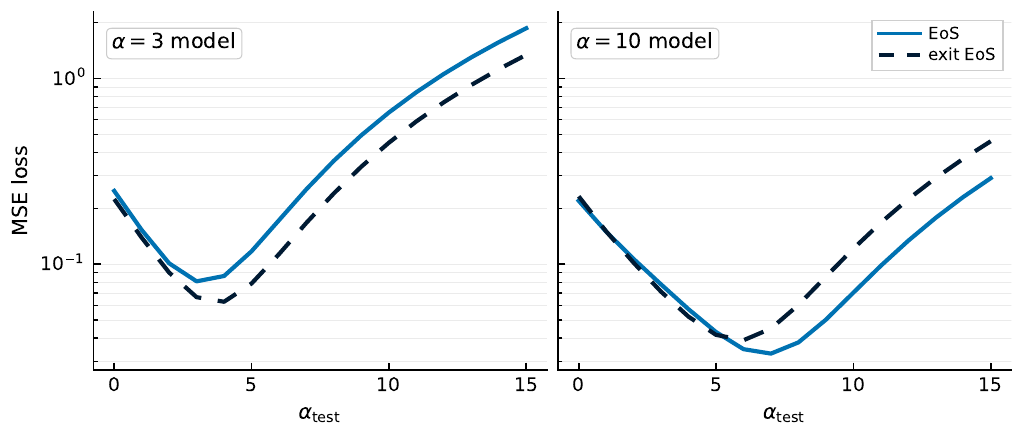}
    \caption{\textbf{EoS steers generalization toward the dominant group.} Test MSE on input-outliers across $\alpha_{\text{test}}$ (after $t^{**}$). \textbf{Left} ($\alpha=3$, boundary dominates $\vone$): no OOD advantage. \textbf{Right} ($\alpha=10$, input-outlier dominates $\vone$): EoS improves OOD performance at large $\alpha_{\text{test}}$. Single seed.}
    \vspace{-0.1cm}
    \label{fig:alpha-ood}
\end{figure}

\section{Discussion}
\label{sec:discussion}

We find that the \eos{} stability constraint acts as an inductive bias, not merely an implicit regularizer. Rather than selecting among equivalent minimizers, it determines which subsets of the data are optimized, as a function of data geometry and loss. Concretely, EoS optimizes on the subset with the largest curvature influence, and the functional consequence depends on which subset dominates.

Comparing the $\alpha =3$ and $\alpha =10$ models illustrates how data geometry changes the functional effect of \eos{}. When boundary points dominate $\vone$, \eos{} focuses optimization near the decision boundary and can improve adversarial robustness; when input-outliers dominate, it instead shifts optimization toward distributional tails and can improve extrapolation along the outlier direction (Appendix~\ref{app:35}). Thus, hyperparameters typically viewed as controlling convergence---such as learning rate or loss---may also influence which subset of the data, and hence which functional property of the model, is emphasized.

This perspective offers a possible explanation for divergent findings in the flatness literature. The benefit of low or controlled sharpness is not determined by scalar flatness alone, but by which subset captures the top curvature direction during training. Different empirical settings may therefore induce different functional outcomes, depending on whether boundary points, input-outliers, or other subgroups dominate $\vone$. This yields a testable prediction: changing the subgroup that dominates $\vone$ should change the direction of the resulting robustness or generalization effect.

\paragraph{Limitations and scope.}
Our experiments focus on relatively small models and low-resolution datasets because tracking \eos{} requires repeated estimation of the top Hessian eigenvalue through Hessian--vector products, with costs increasing with model and dataset size. This limitation is common in the \eos{} literature \citep{cohen2021gradient,andreyev2024edge}, whose experimental setups we build upon. Our multiclass experiments extend beyond binary CIFAR-10 to ImageNet32 (Appendix~\ref{app:multiclass}), but also highlight the limitations of single-mode analysis as the number of classes increases. Extending the mechanism to multiple interacting curvature directions remains future work. We also focus on full-batch gradient descent, as extending the subset-level self-stabilization analysis \citep{damian2023selfstabilization} to mini-batch training requires a corresponding theoretical treatment of stochastic dynamics. Our prototype taxonomy is defined in pixel space, which relates most directly to gradient geometry in MLPs. In convolutional architectures, the curvature-influence score $(\nabla\ell_k\cdot v_1)^2$ remains informative, but the dominant subgroup can change as learned features alter the mapping from inputs to gradients. Extending the taxonomy to representation space is therefore a natural direction. Finally, our robustness and generalization findings remain preliminary and warrant broader evaluation across architectures, seeds, and distribution shifts.

\section{Conclusion}

The edge of stability is not a uniform optimization phenomenon, but a selective one: the stability constraint redistributes learning unevenly across the training distribution, accelerating progress on some subsets while suppressing others. We find that two properties jointly determine which subsets benefit: directional alignment with the top Hessian eigendirection and persistence of gradient signal under the loss function. Removing either component---by disrupting directional coherence or by inducing gradient saturation---eliminates or shifts the advantage to different subsets. Importantly, by characterizing how a curvature constraint differentially shapes learning across the data distribution, this work provides a step toward connecting implicit regularization in parameter space with inductive bias over examples.

\section*{Acknowledgements}
PB and TP acknowledge support from the McGovern Institute for Brain Research, and TP additionally acknowledges support from the McDermott Endowed Chair. The authors thank Arseniy Andreyev, Jeremy Cohen, and Alex Damian for fruitful discussions.

\newpage
\bibliographystyle{unsrt}
\bibliography{references}

\newpage
\appendix
\section{Experimental setup}
\label{app:details}

\subsection{Architecture} We use a fully connected MLP that flattens each input image to a vector, then applies two hidden linear layers of width $512$ with ReLU activations, followed by a final linear classifier to 2 output classes.

\subsection{Training procedure and hyperparameters} Our primary experiments use full-batch GD, which is deterministic given a fixed initialization. We varied the initialization seed to obtain different optimization trajectories. Five pre-determined identical seeds were used in all plots shown in main text. An initial scaling of 0.2 was applied to all runs. 

\subsection{Branching Intervention Implementation}
We isolate the effect of a learning-rate drop at the edge of stability by running each configuration as a matched pair:
\begin{itemize}[leftmargin=*]
\item a \emph{baseline} trained at $\eta$ for the full step budget;
\item a \emph{fork} that follows the baseline trajectory up to a fork step $t^*$ and then applies a single scheduled drop to $\eta'=c\eta$ (we use $c = 0.5$).
\end{itemize}
The fork step $t^*$ is defined as the first logged step at which $\lambda_{\max}(\nabla^2 \mathcal{L})$ crosses $2/\eta$, the EoS threshold for full-batch GD. We utilize a learning rate of $0.01$ for all plots.

Because GD is deterministic given a fixed initialization and has no optimizer state beyond the weights, the fork is implemented by resuming from a checkpoint of the baseline taken just before $t^*$ and continuing at $\eta$ up to $t^*$, at which point the drop is applied. All other settings are identical to the baseline---dataset and initialization seeds, architecture, loss, batch size, step budget, and the fixed prototype subset.

Both branches log all diagnostics on a uniform every-32-steps grid spanning the full training window, ensuring the baseline and fork are sampled at identical step indices so their per-prototype curves can be compared index-for-index without interpolation. The only quantity that differs between a baseline and its fork is the scheduled $\eta \to c\eta$ at step $t^*$. This makes the fork a minimal counterfactual for the learning-rate drop and supports the causal language in Section~\ref{sec:experiments}. Both the fork and the baseline are compared via distance traveled $\eta \cdot \text{step}$ to provide a valid comparison of speed across different learning rates.

\subsection{\eos{} detection and timing}

The effect size depends on when $t^*$ falls relative to the onset of the second EoS regime. When $t^*$ lies well after the sharpness plateau is established, the post-fork branches converge to similar outcomes and the trade-off is attenuated. This is consistent with the mechanism: the intervention acts by redirecting a constraint-shaped trajectory, so once that shaping has occurred, it has less leverage.

The logged $t^*$ has two sources of small offset relative to the true EoS onset. First, the every-32-steps logging grid limits detection resolution: $t^*$ can only be flagged at a logged step, so the actual crossing may have occurred up to 32 steps earlier. 

Second, the cubic term of the Taylor expansion around a point adds a jitter around $2/\eta$ \citep{chen2023beyond}. Specifically, the period-2 orbit of GD exists for $\eta \in \left( {2}/{f''(\bar{x})}, {2}/({f''(\bar{x}) - \epsilon \cdot f^{(3)}(\bar{x}))} \right)$, so the bifurcation is governed by the curvature at the orbit, $f''(\bar{x}) - \epsilon \cdot f^{(3)}(\bar{x})$, rather than the curvature at the minimum, $f''(\bar{x})$, which our $2/\eta$ crossing criterion targets. Together, these two effects can cause the logged $t^*$ to lead or lag the visible onset of oscillation by a small number of steps. This is consistent with the small offsets visible in some figures and does not affect the branching design: both branches share the trajectory up to $t^*$, so the post-fork comparison is unaffected by a few-step mistiming in the detection of EoS onset.

\subsection{Compute and Codebase}
\textbf{Compute.} All experiments ran on a single NVIDIA L40S GPU (48 GB) on an internal academic SLURM cluster, with 8 CPU cores and 64 GB of system RAM per job. A 10,000-step run takes approximately 10 minutes for the MLP and 30 minutes for CNN/ResNet models. The number of GPU-hours is on the order of 15 for all reported experiments. 

\textbf{Codebase.}
Our code is publicly available at \href{https://github.com/anakhag07/edge-of-stability-and-generalization}{this repository}. Our implementation builds on the \href{https://github.com/arseniqum/edge-of-stochastic-stability}{open-source codebase} of \citep{andreyev2024edge}, which provides the CIFAR-10 training pipeline and curvature and sharpness logging infrastructure (Apache 2.0; \citealp{eoss}). We extend this framework with (i) a prototype taxonomy and synthetic outlier construction, (ii) an EoS branching intervention, (iii) per-group alignment diagnostics (gradient norms and cosine alignment), and (iv) per-subset loss and sharpness metrics.

\section{Related Work}

Our work connects four threads: training dynamics at \eos{}, sample-centric analyses of which examples drive learning, curvature-based implicit regularization, and the shatterability of data by ReLU activation regions. We argue these are linked: the directional structure at \eos{} governs how optimization effort is distributed across the data, connecting where the optimizer moves in parameter space to what it learns.

\subsection{Training Dynamics at the Edge of Stability} 

\paragraph{Onset of EoS.} \eos{} occurs when the top Hessian eigenvalue grows along the optimization trajectory until it reaches the stability threshold $2/\eta$ \citep{xingwalk2018,jastrzebski2019relation, jastrzebski2020break-even, cohen2021gradient}. Cohen et al. \citep{cohen2021gradient} termed this growth phase progressive sharpening. In full-batch gradient descent, this gives rise to the edge of stability, where sharpness saturates near $2/\eta$ while the loss continues to decrease \citep{cohen2021gradient}. Recent work has shown that \eos{}-like phenomena extend beyond this setting, to adaptive optimizers, stochastic training, and momentum dynamics \citep{cohen2023adaptive, andreyev2024edge, andreyev2026momentum}, and that saturation levels may be different in case of regularization \citep{saether2026does}. Our goal is complementary: we focus on the cleanest setting for \eos{}, full-batch gradient descent, and use this controlled regime to show that the global stability constraint is also a selective mechanism over the data distribution.
\paragraph{Self-stabilization at EoS.} A mechanistic account is provided by Damian et al., where oscillations along the top Hessian eigenvector $\vone$ bound curvature via self-stabilization \citep{damian2023selfstabilization}. Specifically, these oscillations can feed back through higher-order derivatives to reduce sharpness, yielding an implicit projected dynamics near the stability boundary \citep{damian2023selfstabilization}. Prior work characterizes \eos{} through global dynamics—fast oscillation along $\vone$ and slow sharpness-reducing drift along flat directions \citep{zhu2023understanding, lyu2022understanding, arora2022understanding}. We instead ask which examples contribute to these dynamics, revealing a selective effect over the data distribution that scalar sharpness obscures.

\subsection{Sample-Centric Learning} A parallel literature studies which individual examples drive learning \citep{koh2017understanding, feldman2020neural, feldman2020memorize, toneva2019empirical, swayamdipta2020dataset, paul2021deep, pezeshki2021gradient, rosenfeld2024outliers, sorscher2022beyond}. Rare or atypical examples can disproportionately influence generalization: memorizing them is necessary for near-optimal generalization on long-tailed distributions \citep{feldman2020neural, feldman2020memorize}. Toneva et al. \citep{toneva2019empirical} also track forgetting events and find rare examples are repeatedly forgotten and relearned. Other works score rare examples by difficulty throughout training and reveal ambiguous examples (those which are flipped throughout training) are key for out-of-distribution generalization \citep{swayamdipta2020dataset}. Rare examples can also dominate gradient signal and sharpening dynamics \citep{paul2021deep, pezeshki2021gradient, rosenfeld2024outliers}. In these approaches, example difficulty is defined relative to the model state. We instead define example groups directly from the data distribution and ask how optimization treats them, shifting the question from which examples are hard to which are structurally favored. Other studies \citep{paul2021deep} also introduce GraNd (per-example gradient norm) and EL2N (per-example error-vector norm) as scores that rank training examples by influence on optimization. Our curvature-influence score $(\nabla\ell_i\!\cdot\!\vone)^{2} = \|\nabla\ell_i\|^{2}\cos^{2}\theta_i$ adds the $\cos^{2}\theta_i$ factor on top of GraNd's $\|\nabla\ell_i\|^{2}$. The $\cos^{2}$ factor is what makes the score predictive of the \eos{} beneficiary's identity, and it is invisible to GraNd and EL2N. 

\subsection{Implicit regularization through curvature.} 

Curvature has been widely linked to generalization: sharp minima correlate with worse performance, while smaller batch or higher learning rates implicitly bias toward flatter solutions \citep{foret2021sharpness, lewkowycz2020large, jastrzebski2018three, hochreiter1997flat, keskar2017large, jiang2020fantastic, neyshabur2017exploring, wu2017towards}. Other methods make the flatness objective explicit or directly bias the final iterate toward flatter regions, including Sharpness-Aware Minimization and stochastic weight averaging \citep{foret2021sharpness, izmailov2018averaging}. Much of the flatness literature treats curvature as a property of the final solution or of the algorithm's implicit bias. Work on \eos{} studies curvature dynamically along the training trajectory, including its convergence and implicit-regularization effects \citep{arora2022understanding, ahn2022understanding, ahn2023learning, zhu2023understanding, lyu2022understanding}. We add a data-level perspective: at \eos{}, curvature acts selectively across examples, with functional consequences for which subsets generalize and which are robust.

\subsection{Neural Shattering and Data Geometry}

Liang et al. \citep{liang2025neuralshattering} show that stable minima of two-layer ReLU networks can exhibit \emph{neural shattering}, where rarely active, high-magnitude neurons fit small subsets of the data, leading to poor generalization in high dimensions. Follow-up work \citep{liang2026generalization} links this behavior to data geometry: when training examples can be readily isolated by ReLU activation regions, gradient descent below EoS can favor memorization over shared representations. We provide a complementary perspective at EoS, showing that coherent subgroup gradients reinforce one another and accompany the emergence of selectively activated ReLU neurons that fit the subgroup (Appendix~\ref{app:shattering}). The neural-shattering and data-shatterability perspectives are complementary to ours: shared geometry can make a subgroup isolatable by localized ReLU regions, while \eos{} preferentially acts on that subgroup only when its aggregate gradient forms a persistent near-critical mode. A coherent subgroup may be represented through a shared low-rank feature rather than private shattering. Representational availability and \eos{} selection should therefore not be identified.

\subsection{Our positioning}
Our work connects these research directions by showing how the dominant curvature direction at \eos{} shapes which groups of examples benefit from training and which experience suppressed progress. Prior work establishes that small groups of outliers with opposing signals can disproportionately influence sharpening and \eos{} dynamics \citep{rosenfeld2024outliers}. We extend this perspective by introducing a model-independent prototype taxonomy based on initial data geometry and identifying gradient alignment as a key factor in selective optimization. Complementing work on neural shattering \citep{liang2025neuralshattering,liang2026generalization}, our controlled experiments show how coherent subgroup gradients accompany the emergence of selective ReLU neurons before \eos{} onset, after which the \eos{} loss oscillations concentrate on that subgroup (Appendix~\ref{app:shattering}). Together, these results reveal how a global stability constraint can produce unequal optimization outcomes across training examples.

\section{Robustness Across Experimental Settings}
\label{app:robustness}
Our primary results focus on an MLP trained with full-batch gradient descent under both mean-squared error and cross-entropy losses. In this section, we assess robustness across alternative architectures, optimizers, and class pairs, across 3 seeds and with $\eta$ = 0.01. We examine (i) the divergence in sharpness between baseline and exit-from-EoS runs, (ii) the curvature profile of the baseline, and (iii) how this curvature predicts the intervention effect, defined as $\Delta \ell_k = \ell_k^{\text{exit}} - \ell_k^{\text{EoS}}$.

\subsection{Architecture robustness}
The network's architecture determines how input-space geometry maps to gradient-space geometry. In an MLP, which processes raw pixel vectors, centroid distance translates directly into gradient atypicality: pixel-space outliers produce distinctly oriented gradients. Convolutional architectures transform this mapping through learned spatial features, pooling, and normalization, which can compress pixel-space differences that the MLP preserves. As a result, the group with the highest centroid distance need not be the group with the largest curvature influence. 

The predictive quantity $(\nabla\ell_k \cdot v_1)^2$ remains consistent across architectures---what changes is which group achieves the highest value. \\ 

\textbf{CNN.}   

The CNN consists of three convolutional layers with channel widths $64,64,128$, all using $3\times 3$ kernels, stride $1$, no padding, and ReLU activations, with $2\times 2$ max-pooling after the second and third convolutional layers. The resulting feature map is flattened and passed through a fully connected hidden layer of width $512$ with ReLU activation, followed by a linear classifier to $C$ output classes.

\begin{figure} [h]
    \centering
    \includegraphics[width=1\linewidth]{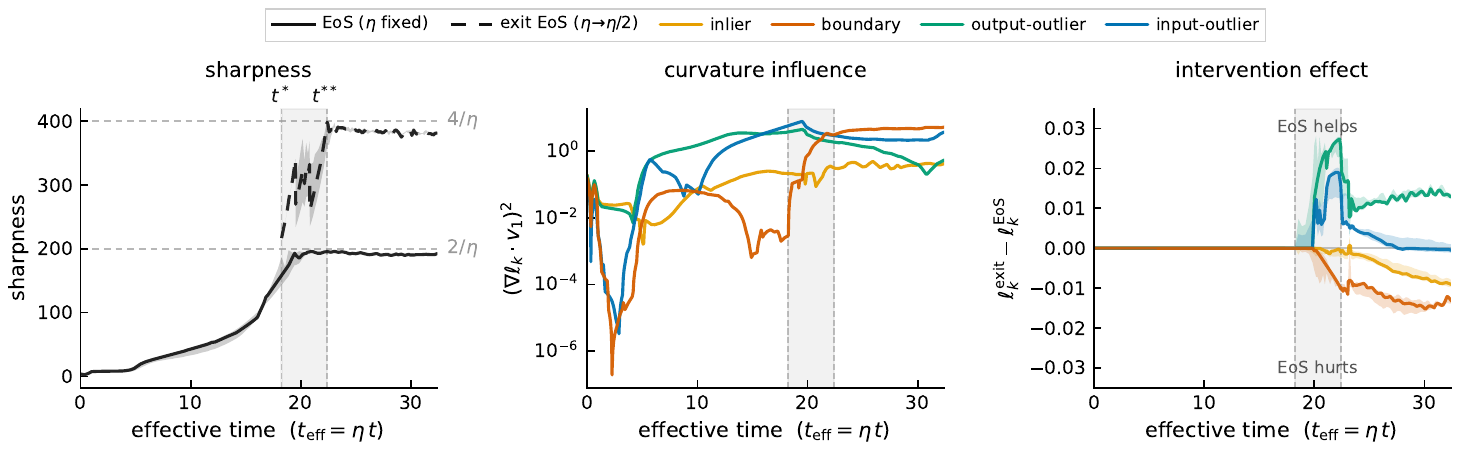}
    \caption{GD CNN MSE. Curvature influence is comparable for output-outliers and input-outliers, both benefit.}
    \label{fig:cnn_cls1v9}
\end{figure}

\FloatBarrier

\textbf{ResNet.}  

We use a batch-normalization-free ResNet-14 with an initial $3\times 3$ convolutional stem of width $16$, followed by three residual stages with channel widths $16,32,64$ and block counts $[2,2,2]$. Each residual block contains two $3\times 3$ convolutions with ReLU activations and identity BatchNorm replacements; downsampling occurs in the first block of stages 2 and 3, followed by global average pooling and a final linear classifier to $C$ output classes.

\begin{figure} [h]
    \centering
    \includegraphics[width=\linewidth]{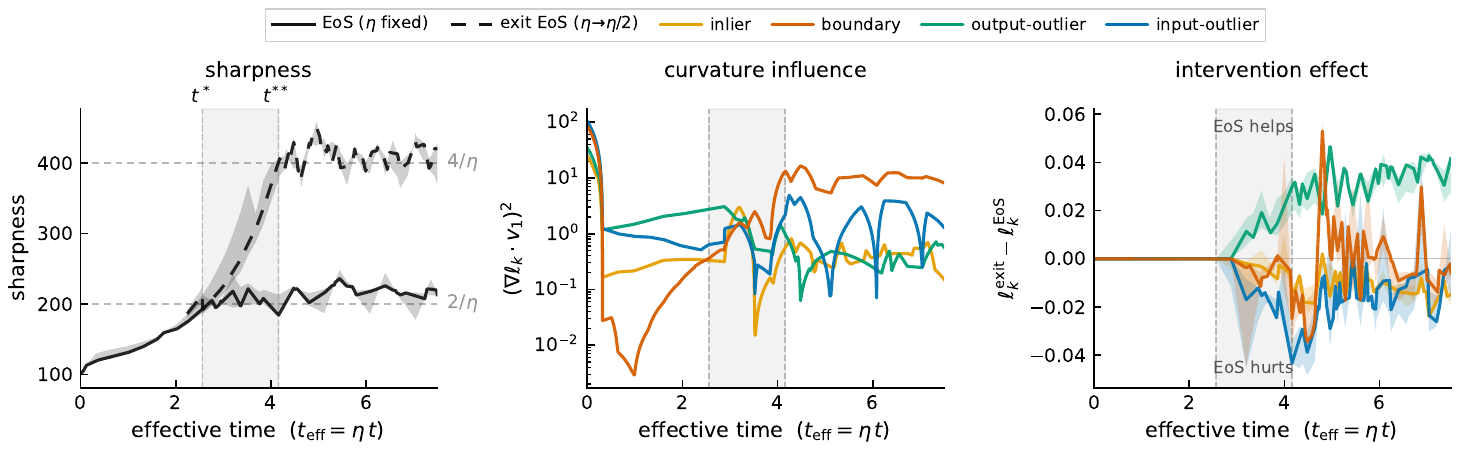}
    \caption{GD ResNet MSE. Curvature influence is highest for output-outliers, and it is the primary beneficiary of EoS.}
    \label{fig:resnet_cls1v9}
\end{figure}

\FloatBarrier

\subsection{Optimizer robustness}

We primarily focused on optimizers which maintained a fixed curvature landscape. We did not extend this analysis to adaptive optimizers such as Adam \citep{adam} as it reshapes the curvature landscape at each step. Instead, stochastic gradient descent (SGD) and full-batch gradient descent with momentum add randomness and acceleration, respectively, to each step the optimizer takes down the loss landscape. With a large batch (Figure~\ref{fig:sgd}), the group with the highest curvature influence (input-outliers) is correspondingly the primary beneficiary of EoS. Similarly for Figure \ref{fig:gdm}, input-outliers have the highest curvature influence, and the intervention benefits them. \\

\textbf{SGD} (batch size=128).
\begin{figure} [h]
    \centering
    \includegraphics[width=1\linewidth]{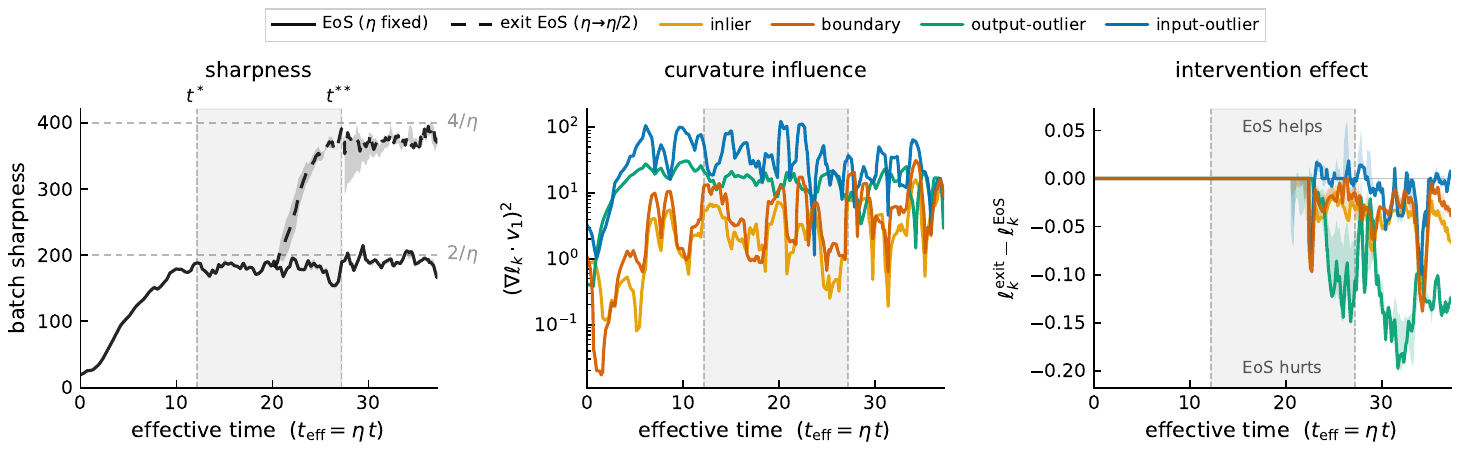}
    \caption{SGD MLP MSE. Curvature influence is highest for input-outliers, and it is the primary beneficiary of EoS.}
    \label{fig:sgd}
\end{figure}

\FloatBarrier

\textbf{GD with Momentum} ($\beta$ = 0.9).
\begin{figure} [h]
    \centering
    \includegraphics[width=1\linewidth]{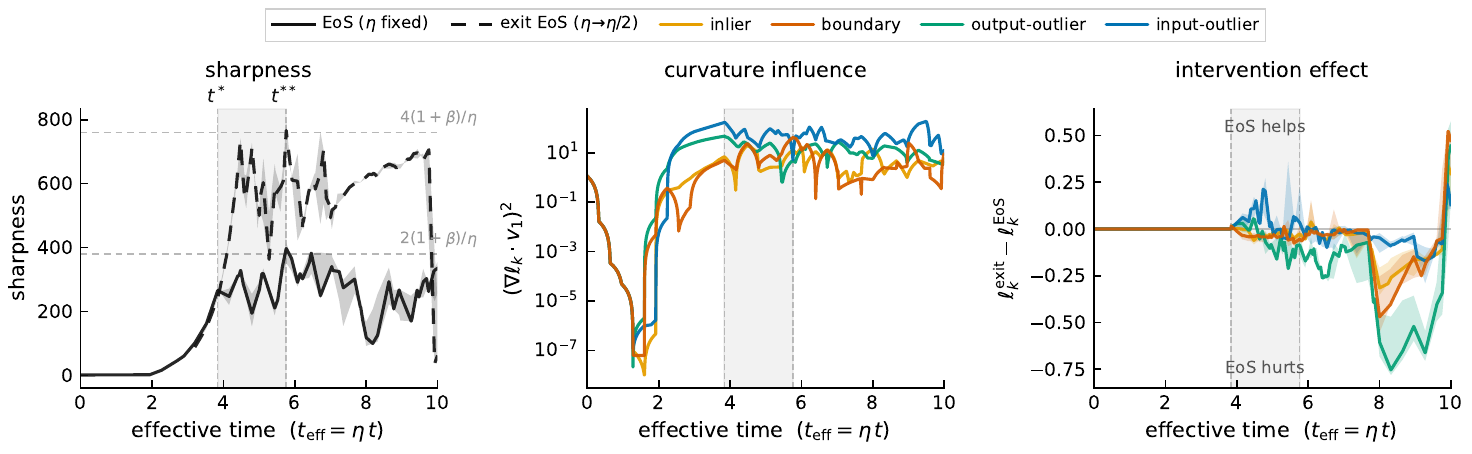}
    \caption{GD Momentum MLP MSE. Curvature influence is highest for input-outliers, and it is the primary beneficiary of EoS. Momentum implemented as in \citep{polyak1964some} and EoS threshold implemented as described for large-batch momentum in \citep{andreyev2026momentum}.}
    \label{fig:gdm}
\end{figure}

\FloatBarrier

\subsection{Class pair robustness}
\label{app:35}
\textbf{A harder classification task.}
On the closer pair (3,5), $\alpha$ = 3 is no longer sufficient to make input-outliers the most atypical subset; boundary points dominate atypicality instead (Figure~\ref{fig:piano35}). The alignment principle predicts that boundary points should therefore capture the \eos{} advantage on this pair, and they do (Figure~\ref{fig:multi35}).

\begin{figure} [h]
    \centering
    \includegraphics[width=0.7\linewidth]{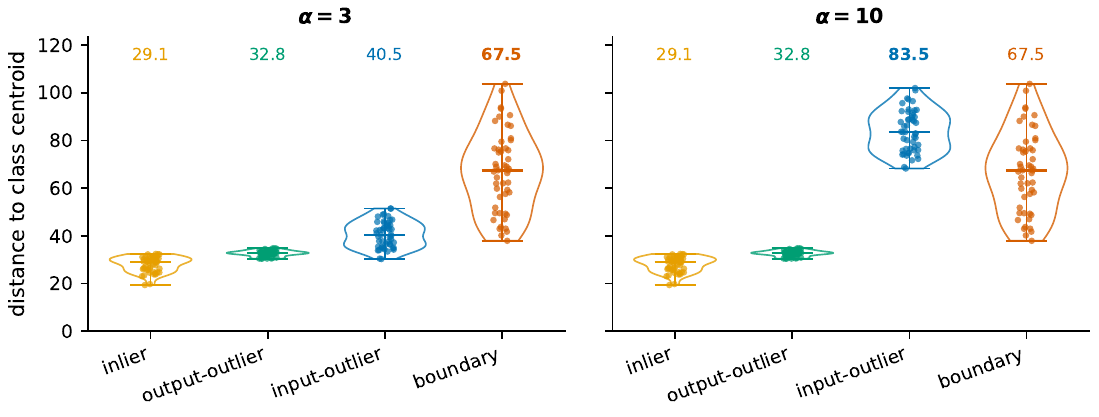}
    \caption{Distribution of centroid distance by prototype subgroup on the (3,5) class pair. Boundary points have the largest median distance under $\alpha=3$ but input-outliers have the largest median distance under $\alpha=10$.}
    \label{fig:piano35}
\end{figure}

\textbf{Ablation on $\alpha$.}
At $\alpha=3$, boundary points are the primary beneficiaries at \eos{}. At $\alpha=10$, input-outliers become the most distant from their class centroids and the \eos{} advantage transfers to them.

\begin{figure} [h]
    \centering
    \includegraphics[width=1\linewidth]{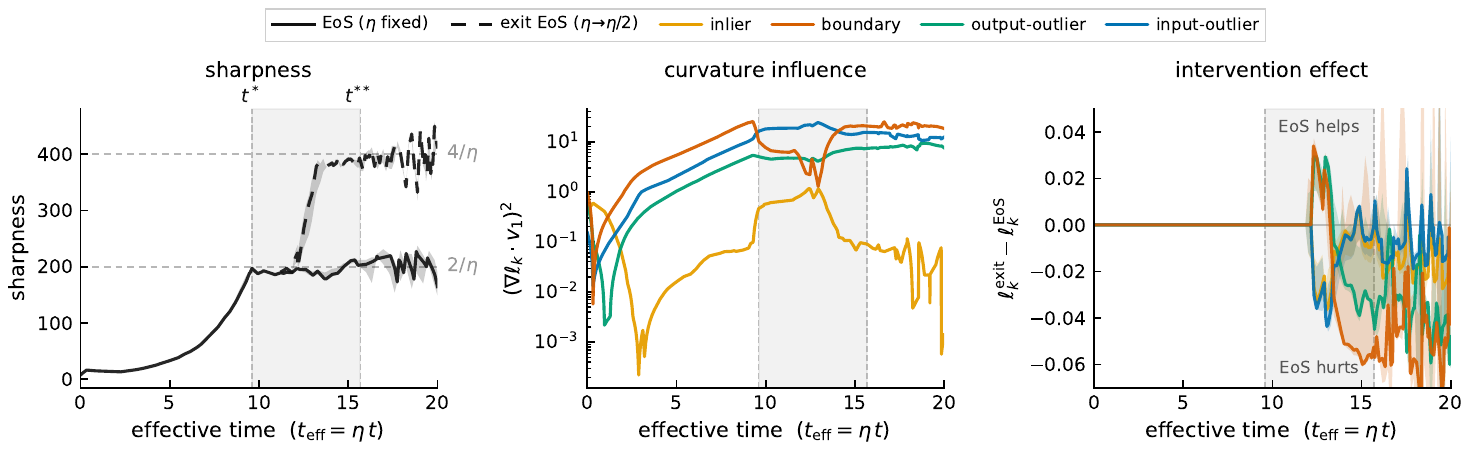}
    \caption{Curvature influence under input-outlier construction ($\alpha = 3$). The boundary group has the highest curvature influence and is the primary beneficiary of EoS.}
    \label{fig:multi35}
\end{figure}

\FloatBarrier

\begin{figure} [h]
    \centering
    \includegraphics[width=1\linewidth]{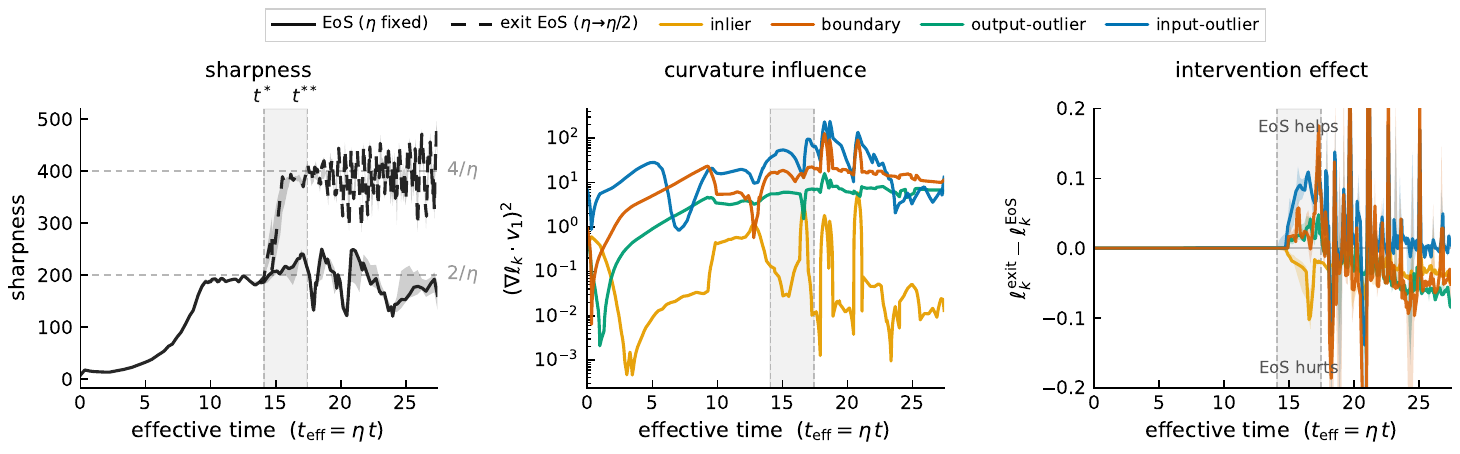}
    \caption{Curvature influence under input-outlier construction ($\alpha = 10$). The input-outlier group has the highest curvature influence and is the primary beneficiary of EoS.}
    \label{fig:cls3v5_alpha10}
\end{figure}

\FloatBarrier
\section{Extension to ImageNet32 and Multiclass Settings}
\label{app:multiclass}
\paragraph{Experimental setup.}
We extend the protocol to $K\in\{2,3,5,10\}$ class subsets of ImageNet32 \citep{chrabaszcz2017downsampled} (MLP, MSE loss, full-batch gradient descent, three seeds per $K$), with the identical \eos{} to exit-\eos{} fork: $\eta=0.01$, reduced to $\eta_{\mathrm{low}}=0.005$ at the baseline's onset $t^*$. The prototype construction generalizes per class: inlier and boundary pools are drawn per class, output-outliers are relabeled to the nearest alternative class centroid, and input-outliers are extrapolated away from the nearest alternative class. Both reduce exactly to the binary construction at $K=2$, so the main-text setting is a special case of this design. All runs exhibit standard \eos{} dynamics and re-equilibrate after the drop, with onset arriving later as $K$ grows, so the branching intervention is well defined at every class count.

\paragraph{The single-mode trade-off weakens with $K$.}
On ImageNet32, the trade-off persists at $K=2$ and $K=3$: input-outliers benefit from staying at \eos{} ($\Delta\ell_k > 0$ within the intervention window), while every other group benefits from exiting. At $K=5$ and $K=10$, every group benefits from exiting and the trade-off weakens (Figure~\ref{fig:in32_ksweep}). As $K$ grows, multiple curvature directions compete for the instability and absolute alignment with $\vone$ dilutes across class groups: at $K=10$, output-outliers carry roughly $25\%$ of the gradient norm but only $3\%$ of the $\vone$-alignment. By the late global onset, input-outliers have largely lost their dominant coupling to $\vone$ and have already been fit, so alignment no longer coincides with a non-saturated residual, which the mechanism predicts is necessary for the trade-off (Appendix~\ref{app:single-mode}). A multiclass extension should therefore track a near-critical subspace rather than a single $\vone$.

\begin{figure}[h]
  \centering
  \includegraphics[width=\linewidth]{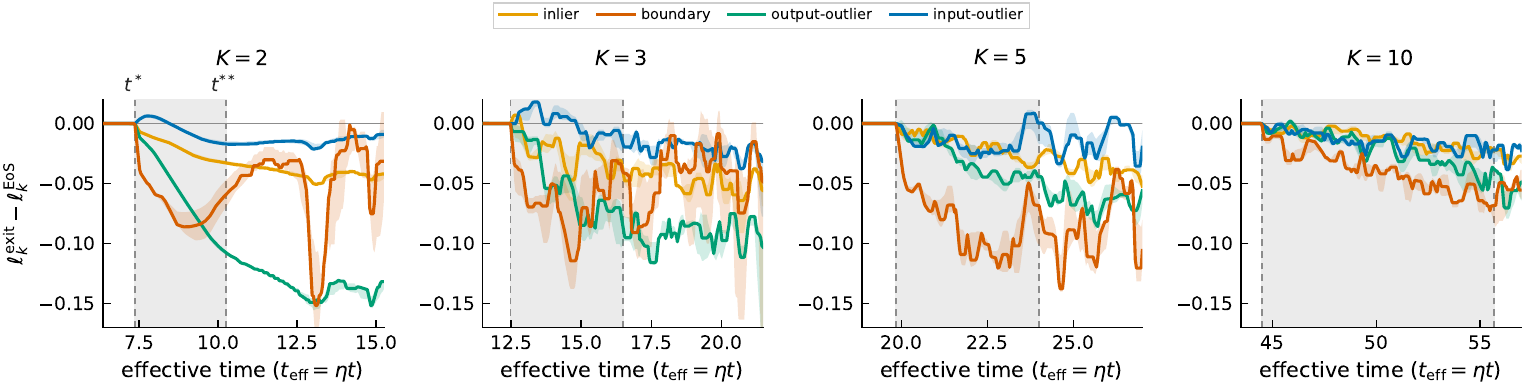}
  \caption{Multiclass ImageNet32 experiments with MSE loss, varying the number of classes $K$. While input-outliers benefit from staying at EoS at $K=2$ or $K=3$, no comparable advantage emerges at $K=5$ or $K=10$.}
  \label{fig:in32_ksweep}
\end{figure}

\FloatBarrier

\section{Input-Outlier Construction Ablation}
\label{app:coherent_random}

The conceptual schematic for the directional perturbation is shown in Figure~\ref{fig:coherent_random_schema}. Figure~\ref{fig:piano19} confirms that geometric atypicality of input-outlier is preserved under the random-direction control.

\begin{figure} [h]
    \centering
    \includegraphics[width=0.9\linewidth]{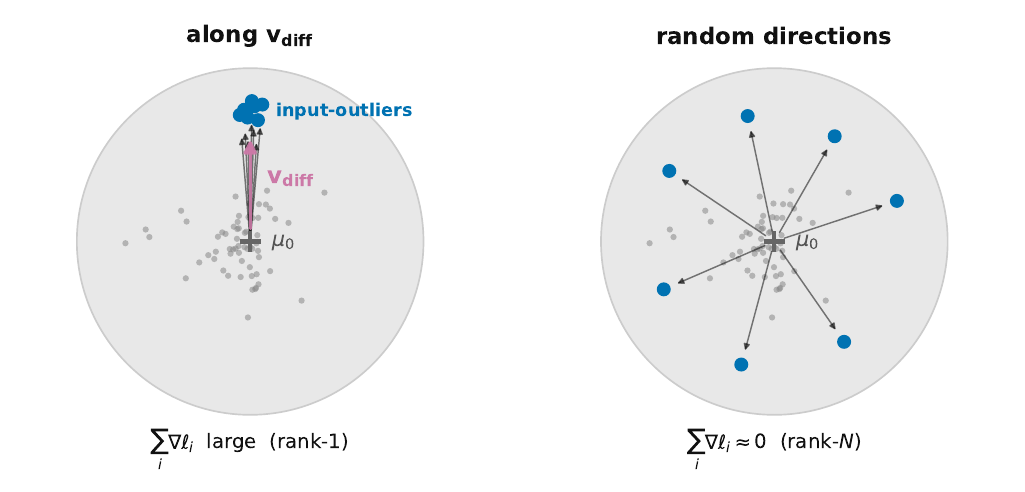}
    \caption{Schematic for coherent vs.\ incoherent input-outlier construction. \textbf{Left:} input-outliers displaced along a shared direction $v_{\mathrm{diff}}$; per-example gradients reinforce, producing a dominant curvature contribution along one direction. \textbf{Right:} same displacement distance but in random orthogonal directions; per-example gradients partially cancel, producing a diffuse contribution across many eigenvectors. Seed points, labels, and centroid distances are identical across conditions.}
    \label{fig:coherent_random_schema}
\end{figure}

\begin{figure} [h]
    \centering
    \includegraphics[width=0.7\linewidth]{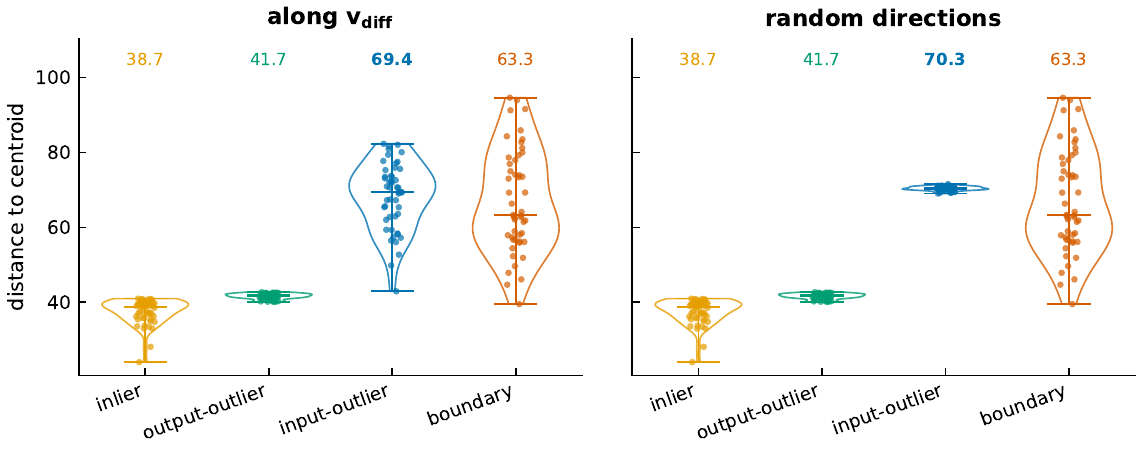}
    \caption{Centroid distance is preserved under the random-direction control (median 69.4 vs.\ 70.3 for along-$v_{\text{diff}}$ and random-direction outliers, respectively). Geometric atypicality is conserved.}
    \label{fig:piano19}
\end{figure}
\FloatBarrier

\section{Gradient Mechanism of Data Selection and Connection to Neural Shattering}
\label{app:shattering}

This appendix examines how coherent input geometry produces reinforcing per-neuron gradients and how these gradients relate to subgroup-selective ReLU activation. In a controlled sandbox, coherent input-outliers exhibit strong gradient alignment and substantially better fitting than equally displaced random-direction outliers, although sparsely activated neurons arise in both conditions. This provides a neuron-level perspective on the coherence mechanism of Lemma~\ref{lem:coherence}, connects our findings to neural shattering \citep{liang2025neuralshattering,liang2026generalization}, and complements the full-scale comparison in Appendix~\ref{app:coherent_random}.

\paragraph{Per-neuron gradient analysis.}
For a two-layer ReLU network $f(x)=\sum_j a_j\sigma(w_j^\top x+b_j)$ trained with MSE, the hidden-weight gradient is
\[
\frac{\partial L}{\partial w_j}
= \frac{2a_j}{n}\sum_i r_i\,\mathbf{1}\{w_j^\top x_i+b_j>0\}\,x_i,
\qquad r_i=f(x_i)-y_i,
\]
so each neuron receives signal only from examples in its active halfspace, weighted by their residuals. Let $S$ contain $k$ coherent input-outliers $x_i=s_i+\alpha u$, where $u=\pm v_{\mathrm{diff}}$ is the per-class outward displacement direction (Section~\ref{sec:prototypes}). Writing $I_{ij}=\mathbf{1}\{w_j^\top x_i+b_j>0\}$, the subgroup contribution to neuron $j$ decomposes as
\[
g_j^S
=\frac{2a_j\alpha}{n}
\left(\sum_{i\in S}r_iI_{ij}\right)u
+\frac{2a_j}{n}\sum_{i\in S}r_iI_{ij}s_i.
\]
When gates and residual signs are consistent across the subgroup, the displacement terms reinforce into an $O(k\alpha)$ pull along $u$, whereas independent random displacements of equal magnitude contribute only $O(\sqrt{k}\,\alpha)$ on average, with no shared direction. Because ReLU gating depends on the inputs, these scalings are heuristic. A neuron whose halfspace isolates $S$ receives nearly pure subgroup signal, and the displacement-dependent part of its Gauss--Newton curvature contribution ($\nabla_{w_j}f=a_jI_{ij}x_i$) scales as $\alpha^2$ for fixed gates. Whether gradient descent reaches such configurations is the empirical question the sandbox examines.

\paragraph{Sandbox experiment.}
We test this mechanism on a 2-D two-Gaussian binary MSE task ($n=208$) with a two-layer ReLU network ($m=50$) and full-batch GD, displacing $k=8$ class-0 seeds either coherently along $\alpha u$ with $\alpha=2$ or along equal-norm independent random directions (the 2-D analogue of the control in Appendix~\ref{app:coherent_random}). In both runs, $\lambda_{\max}$ rises quickly to just below the stability threshold and first crosses $2/\eta$ at $t^*$ (median $\lambda_{\max}/(2/\eta)=0.999$ and $0.998$). Selective activation patterns form before $t^*$: one neuron is selective from initialization, and the other completes its migration shortly before onset. Pronounced loss oscillations then appear on the already-fitted subgroup.

Figure~\ref{fig:relu_shattering} summarizes the results. In the coherent run, the mean pairwise cosine of subgroup gradient contributions within a neuron's active set approaches $1$, and two neurons become selective for the subgroup. They activate on $88\%$ and $75\%$ of the input-outliers but on at most $1.5\%$ of the bulk, whereas every other neuron covering the subgroup fires on at least $6.5\%$ of the bulk (Figure~\ref{fig:relu_shattering}, middle). Their halfspaces isolate the subgroup (Figure~\ref{fig:relu_shattering}, left), and their received subgroup gradient satisfies $\lvert\cos(g_j^S,u)\rvert=0.99$. The subgroup is fit to MSE $\approx 10^{-4}$ by $t^*$, below the bulk loss, and during the subsequent oscillatory phase its loss fluctuates over three orders of magnitude while the bulk loss varies by about a factor of two (Figure~\ref{fig:relu_shattering}, right). In the random-direction control, pairwise cosines remain around $0.2$--$0.5$, alignment reaches only $0.60$, and subgroup MSE stays above $3\times10^{-2}$. The control also produces rarely active neurons that capture a few displaced points, but none combines selective activation with reinforced gradients and subgroup fitting.

\begin{figure}[h]
    \centering
    \includegraphics[width=\linewidth]{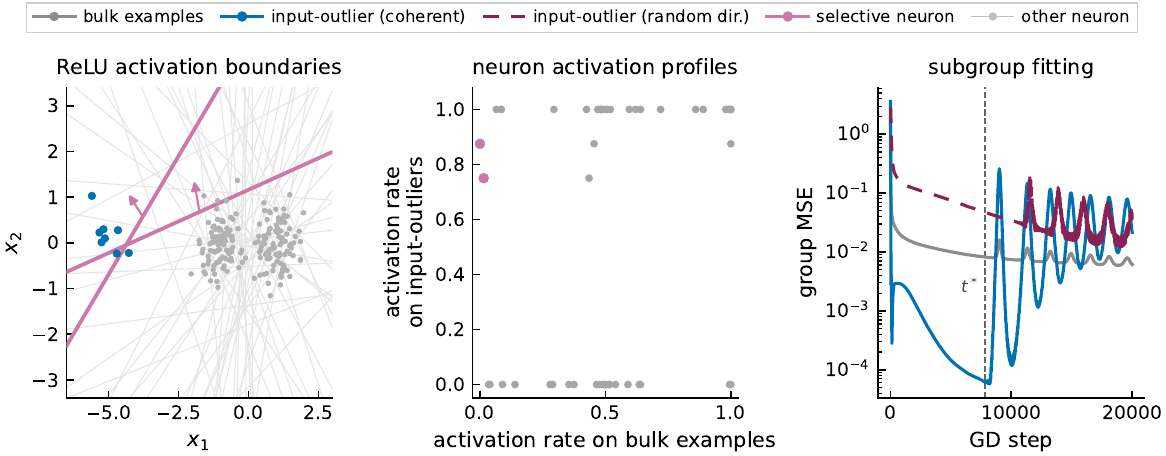}
    \caption{Coherent input-outliers are isolated and fit by subgroup-selective ReLU neurons. \textbf{Left:} Activation boundaries of all $m=50$ neurons, highlighting two whose active halfspaces (arrows) isolate the input-outliers from the bulk. \textbf{Middle:} Neuron activation rates on input-outliers versus bulk examples. The two selective neurons activate on most outliers but at most $1.5\%$ of the bulk. \textbf{Right:} Per-group training loss. Coherent outliers are fit below the bulk by \eos{} onset $t^*$, followed by pronounced subgroup loss oscillations, while random-direction outliers remain at higher loss.}
    \label{fig:relu_shattering}
\end{figure}
\FloatBarrier

\paragraph{Relation to neural shattering.}
The selective neurons show the qualitative signatures of neural shattering \citep{liang2025neuralshattering}. They activate rarely ($3$--$4\%$ of all examples), carry effective weight magnitude $\lVert w_j\rVert\lvert a_j\rvert$ about twice the median, and fit a small group of points locally. The two mechanisms are related but not identical. Neural shattering describes many nearly private neurons that each memorize individual points, whereas here a small number of neurons exhibit activation patterns shared across the coherent subgroup, consistent with group-level specialization rather than pointwise memorization. The perspectives are therefore complementary. Data shatterability \citep{liang2026generalization} describes which localized representations are available to the network, while alignment and persistence (Section~\ref{sec:mechanism}, Appendix~\ref{app:theory}) determine which subgroup \eos{} preferentially optimizes.

\section{Natural points}
\label{app:natural}
We verify that the alignment principle holds on natural geometric outliers, with no synthetic displacement applied. Figure~\ref{fig:corr_progress} plots the Spearman rank correlation coefficient between centroid distance and per-example alignment $\cos^2(\nabla\ell_i, v_1)$ across training: the correlation is near zero during progressive sharpening and rises at \eos{} onset. The Spearman coefficient measures rank agreement; lines show the line of best fit for visual clarity.

Figure~\ref{fig:corr_points} shows the underlying per-example trajectories. Atypical points (red) exhibit high alignment that peaks just before \eos{} onset and then declines, while typical points (blue) remain weakly aligned over the same interval. This relationship is monotonic with respect to distance from the centroid near \eos{} onset. Later in training, the ordering reverses, with typical points eventually exceeding atypical ones in alignment.

\begin{figure} [h]
    \centering
    \includegraphics[width=0.9\linewidth]{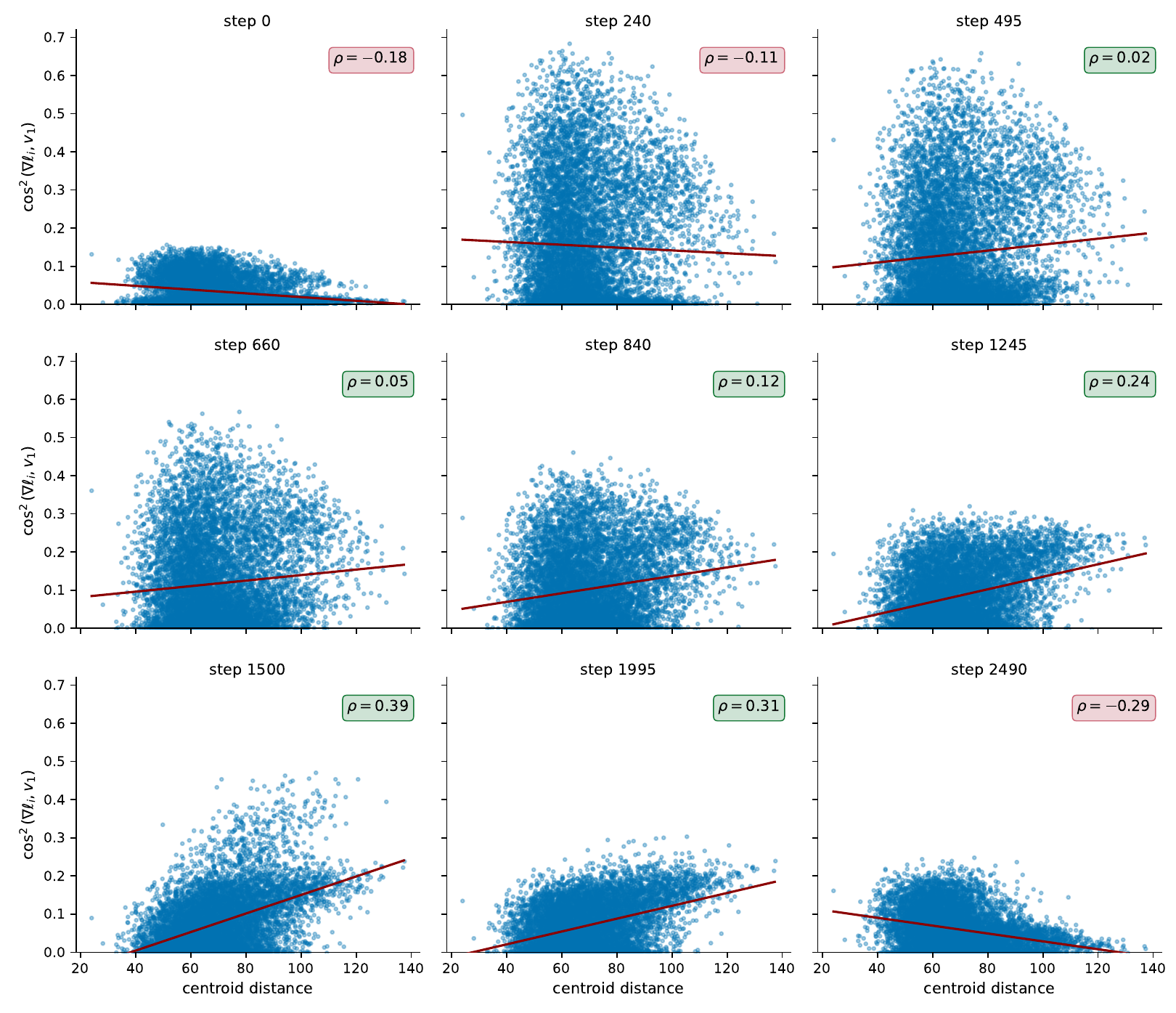}
    \caption{\textbf{Top:} early in training, the correlation starts negative and increases. \textbf{Middle:} at \eos{} onset, the correlation increases monotonically. \textbf{Bottom:} the correlation peaks when large-amplitude oscillations along $\vone$ develop, then decreases monotonically and turns negative.}
    \label{fig:corr_progress}
\end{figure}
\FloatBarrier

\begin{figure} [h]
    \centering
    \includegraphics[width=0.8\linewidth]{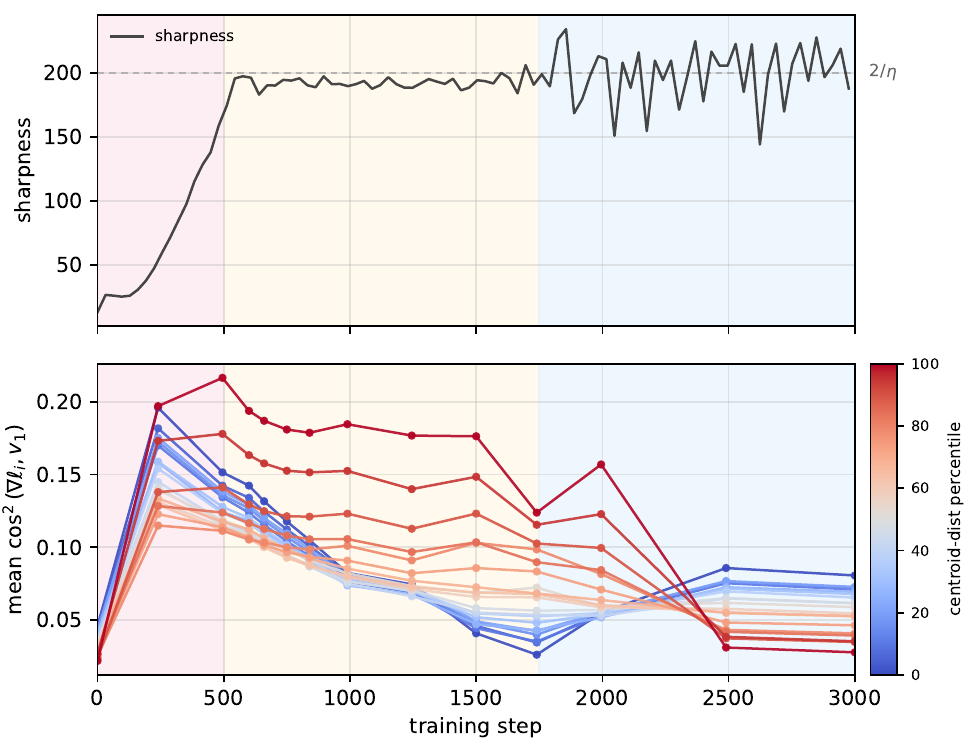}
    \caption{\textbf{Top:} correlation rises at \eos{} onset and peaks with sharpness oscillations. \textbf{Bottom:} alignment vs. centroid-distance percentile over training. Atypical points peak near \eos{} (step 660), then decline, and dominance shifts to typical points later in training.}
    \label{fig:corr_points}
\end{figure}
\FloatBarrier

\section{Theory: EoS Self-stabilization as a Subset Selector}
\label{app:theory}

This appendix gives the local mathematical mechanism behind the subset-level effects measured in the main text. The starting point is the cubic self-stabilization model of \citep{damian2023selfstabilization}: at the edge of stability, gradient descent admits a slow, cycle-averaged description as projected gradient descent on the active sharpness constraint $S(\theta) = 2/\eta$. Under this assumption, we derive a first-order decomposition of the branch differential into two distinct contributions and isolate the two mechanisms tested experimentally---directional alignment and gradient persistence.

\subsection{Notation and assumptions}
\label{app:theory-notation}
Let $L:\mathbb{R}^d \to \mathbb{R}$ denote the full empirical training loss, and let \[ \ell_k(\theta) := \frac{1}{|P_k|} \sum_{i \in P_k} \ell(f_\theta(x_i), y_i) \] be the loss restricted to prototype group $k$. Let $H(\theta) := \nabla^2 L(\theta)$, $S(\theta) := \lambda_{\max}(H(\theta))$, and $\mathbf{v}_1(\theta)$ denote the Hessian, sharpness, and unit top Hessian eigenvector respectively (assumed simple throughout). We compare two branches forked from $\theta^* = \theta_{t^*}$ at EoS onset, where $\eta S(\theta^*) \approx 2$. The baseline branch continues at learning rate $\eta$. The exit branch uses learning rate $c\eta$ with $c \in (0,1)$, locally stable at the fork. We use the baseline learning-rate-scaled time \[ \tau := \eta(t - t^*), \] where $t$ denotes the gradient descent iteration index and $t^*$ denotes the EoS onset step at which the branches fork. This scaling gives a continuous-time comparison variable aligned to the baseline branch. Define the cycle-averaged (two-step average) branch differential \[ \bar\Delta\ell_k(\tau) := \bar\ell_k(\theta_{\text{exit}}(\tau)) - \bar\ell_k(\theta_{\text{EoS}}(\tau)), \] where the bar denotes the two-step average that suppresses the $O(\delta)$ phase term along $\mathbf{v}_1$ generated by the period-two EoS oscillation. Define the sharpness-gradient quantities \[ \alpha := -\langle \nabla L, \nabla S \rangle, \qquad \beta := \|\nabla S\|^2, \qquad \delta := \sqrt{2\alpha/\beta}, \] where $\delta$ measures the oscillation amplitude along $\mathbf{v}_1$ in the EoS regime \citep{damian2023selfstabilization}, and progressive sharpening corresponds to $\alpha > 0$. We use the following standing assumptions.

\paragraph{(A1) Smoothness and simple top eigenvalue.} $L \in C^3$ (L is thrice-differentiable) and each $\ell_k \in C^2$ ($\ell_k$ is twice-differentiable) in a neighborhood of $\theta^*$, with simple largest Hessian eigenvalue.

\paragraph{(A2) Local EoS approximation.} The baseline EoS trajectory admits a slow, cycle-averaged description as projected gradient descent on the active sharpness constraint $S(\theta) = 2/\eta$, in the sense of \citep{damian2023selfstabilization}.

\paragraph{(A3) Stable exit.} The exit branch is locally stable in the top direction at the fork, $c\eta S(\theta^*) < 2$.

\paragraph{(A4) Short-window Taylor regime.} Post-fork comparisons are made in a window short enough that gradients and Hessians admit Taylor expansion around $\theta^*$ with $O(\tau^2)$ remainder.

\paragraph{(A5) Cycle averaging.} Subset losses are compared via two-step averaging or short-window smoothing, suppressing the $O(\delta)$ instantaneous phase fluctuation along $u$ at EoS.

\subsection{Projected drift at the sharpness boundary}
\label{app:projected-drift}

\begin{lemma}[EoS slow drift]
\label{lem:eos-drift}

Under (A1)--(A2), the cycle-averaged EoS drift at $\theta^*$ is
\[
    \dot\theta_{\text{EoS}} = -\nabla L - \frac{\alpha}{\beta} \nabla S.
\]
Under (A3), the exit branch in baseline time $\tau$ has leading drift $\dot\theta_{\text{exit}} = -c\nabla L$.
\end{lemma}

\begin{proof}
The tangent space of the active boundary $\mathcal{M} := \{\theta : S(\theta) = 2/\eta\}$ at $\theta^*$ is $\{z : \langle \nabla S, z\rangle = 0\}$. Removing the normal component of $\nabla L$ along $\nabla S$ gives
\[
    \nabla L|_{\mathcal{M}} = \nabla L - \frac{\langle \nabla L, \nabla S\rangle}{\|\nabla S\|^2} \nabla S = \nabla L + \frac{\alpha}{\beta} \nabla S,
\]
so projected-gradient descent yields drift $-\nabla L - (\alpha/\beta)\nabla S$. The exit-branch drift in baseline time follows from (A3): a step of size $c\eta$ in baseline time $\tau = \eta(n - n^*)$ has slope $-c\nabla L$ to leading order.
\end{proof}

The key observation is that the EoS branch is not ordinary gradient descent with oscillations: its slow drift carries an additional component along $-\nabla S$. Each subset feels this additional drift through the inner product $\langle \nabla\ell_k, \nabla S\rangle$.

\subsection{Branch decomposition}
\label{app:main-decomposition}

Define the two local subset scores
\[
    R_k := \langle \nabla\ell_k, \nabla L\rangle, \qquad Q_k := \langle \nabla\ell_k, \nabla S\rangle.
\]
$R_k$ is the ordinary loss-gradient alignment of subset $k$. $Q_k$ is the projection of the subset gradient onto the sharpness-control direction.

\begin{proposition}[Per-subset branch decomposition]
\label{prop:decomposition}
Under (A1)--(A5), for short post-branch times,
\[
    \bar\Delta\ell_k(\tau)
    \;=\;
    \underbrace{(1-c)\, R_k\, \tau}_{\text{learning-rate confounder}}
    \;+\;
    \underbrace{\frac{\alpha}{\beta}\, Q_k\, \tau}_{\text{EoS selector}}
    \;+\; O_k(\tau^2) + O_k(\delta^2).
\]
\end{proposition}


\begin{proof}
All quantities below are evaluated locally near the branch point $\theta^*$ unless otherwise stated. Recall that
\[
    R_k := \langle \nabla \ell_k(\theta^*), \nabla L(\theta^*)\rangle,
    \qquad
    Q_k := \langle \nabla \ell_k(\theta^*), \nabla S(\theta^*)\rangle .
\]
By Lemma~\ref{lem:eos-drift}, the cycle-averaged EoS branch has slow drift
\[
    \dot\theta_{\mathrm{EoS}}
    =
    -\nabla L - \frac{\alpha}{\beta}\nabla S,
\]
while the exit branch, measured in baseline time $\tau$, has leading drift
\[
    \dot\theta_{\mathrm{exit}}
    =
    -c\nabla L.
\]

We first compute the rate of change of the subset loss along the EoS branch. By the chain rule,
\[
    \frac{d}{d\tau}
    \bar\ell_k(\theta_{\mathrm{EoS}}(\tau))
    =
    \left\langle
        \nabla \ell_k(\theta_{\mathrm{EoS}}(\tau)),
        \dot\theta_{\mathrm{EoS}}(\tau)
    \right\rangle .
\]
Using the local EoS drift from Lemma~\ref{lem:eos-drift}, this becomes
\[
    \frac{d}{d\tau}
    \bar\ell_k(\theta_{\mathrm{EoS}}(\tau))
    =
    \left\langle
        \nabla \ell_k(\theta_{\mathrm{EoS}}(\tau)),
        -\nabla L(\theta_{\mathrm{EoS}}(\tau))
        -
        \frac{\alpha(\tau)}{\beta(\tau)}
        \nabla S(\theta_{\mathrm{EoS}}(\tau))
    \right\rangle .
\]
Expanding the inner product gives
\[
    \frac{d}{d\tau}
    \bar\ell_k(\theta_{\mathrm{EoS}}(\tau))
    =
    -
    \left\langle
        \nabla \ell_k(\theta_{\mathrm{EoS}}(\tau)),
        \nabla L(\theta_{\mathrm{EoS}}(\tau))
    \right\rangle
    -
    \frac{\alpha(\tau)}{\beta(\tau)}
    \left\langle
        \nabla \ell_k(\theta_{\mathrm{EoS}}(\tau)),
        \nabla S(\theta_{\mathrm{EoS}}(\tau))
    \right\rangle .
\]
For short post-branch times, Assumption~(A4) allows us to replace the quantities along the branch by their values at $\theta^*$ up to first-order local errors:
\[
    \left\langle
        \nabla \ell_k(\theta_{\mathrm{EoS}}(\tau)),
        \nabla L(\theta_{\mathrm{EoS}}(\tau))
    \right\rangle
    =
    R_k + O_k(\tau),
\]
and
\[
    \left\langle
        \nabla \ell_k(\theta_{\mathrm{EoS}}(\tau)),
        \nabla S(\theta_{\mathrm{EoS}}(\tau))
    \right\rangle
    =
    Q_k + O_k(\tau).
\]
Likewise, $\alpha(\tau)/\beta(\tau)=\alpha/\beta+O(\tau)$ locally. Therefore
\[
    \frac{d}{d\tau}
    \bar\ell_k(\theta_{\mathrm{EoS}}(\tau))
    =
    -R_k
    -
    \frac{\alpha}{\beta}Q_k
    +
    O_k(\tau)
    +
    O_k(\delta^2).
\]
The term $O_k(\delta^2)$ comes from replacing the instantaneous oscillatory EoS trajectory by its two-step cycle average. The leading $O(\delta)$ phase term cancels under the two-step average, leaving only second-order oscillation effects.

Now consider the exit branch. Again by the chain rule,
\[
    \frac{d}{d\tau}
    \bar\ell_k(\theta_{\mathrm{exit}}(\tau))
    =
    \left\langle
        \nabla \ell_k(\theta_{\mathrm{exit}}(\tau)),
        \dot\theta_{\mathrm{exit}}(\tau)
    \right\rangle .
\]
Using $\dot\theta_{\mathrm{exit}}=-c\nabla L$ gives
\[
    \frac{d}{d\tau}
    \bar\ell_k(\theta_{\mathrm{exit}}(\tau))
    =
    -c
    \left\langle
        \nabla \ell_k(\theta_{\mathrm{exit}}(\tau)),
        \nabla L(\theta_{\mathrm{exit}}(\tau))
    \right\rangle .
\]
Again Taylor-expanding around $\theta^*$ over the short window,
\[
    \left\langle
        \nabla \ell_k(\theta_{\mathrm{exit}}(\tau)),
        \nabla L(\theta_{\mathrm{exit}}(\tau))
    \right\rangle
    =
    R_k + O_k(\tau),
\]
so
\[
    \frac{d}{d\tau}
    \bar\ell_k(\theta_{\mathrm{exit}}(\tau))
    =
    -cR_k + O_k(\tau).
\]

We now differentiate the branch differential
\[
    \bar\Delta\ell_k(\tau)
    :=
    \bar\ell_k(\theta_{\mathrm{exit}}(\tau))
    -
    \bar\ell_k(\theta_{\mathrm{EoS}}(\tau)).
\]
Taking a derivative with respect to $\tau$ yields
\[
    \frac{d}{d\tau}\bar\Delta\ell_k(\tau)
    =
    \frac{d}{d\tau}
    \bar\ell_k(\theta_{\mathrm{exit}}(\tau))
    -
    \frac{d}{d\tau}
    \bar\ell_k(\theta_{\mathrm{EoS}}(\tau)).
\]
Substituting the two expressions above,
\[
    \frac{d}{d\tau}\bar\Delta\ell_k(\tau)
    =
    \left[-cR_k + O_k(\tau)\right]
    -
    \left[
        -R_k
        -
        \frac{\alpha}{\beta}Q_k
        +
        O_k(\tau)
        +
        O_k(\delta^2)
    \right].
\]
Simplifying,
\[
    \frac{d}{d\tau}\bar\Delta\ell_k(\tau)
    =
    (1-c)R_k
    +
    \frac{\alpha}{\beta}Q_k
    +
    O_k(\tau)
    +
    O_k(\delta^2).
\]

At the branching time, both branches start from the same parameter value, so
\[
    \bar\Delta\ell_k(0)=0.
\]
Integrating from $0$ to $\tau$ gives
\[
    \bar\Delta\ell_k(\tau)
    =
    \int_0^\tau
    \frac{d}{ds}\bar\Delta\ell_k(s)\,ds.
\]
Using the expression for the derivative,
\[
    \bar\Delta\ell_k(\tau)
    =
    \int_0^\tau
    \left[
        (1-c)R_k
        +
        \frac{\alpha}{\beta}Q_k
        +
        O_k(s)
        +
        O_k(\delta^2)
    \right] ds.
\]
The leading terms are constant in this local expansion, so
\[
    \int_0^\tau
    \left[
        (1-c)R_k
        +
        \frac{\alpha}{\beta}Q_k
    \right] ds
    =
    \left[
        (1-c)R_k
        +
        \frac{\alpha}{\beta}Q_k
    \right]\tau.
\]
The local Taylor error integrates as
\[
    \int_0^\tau O_k(s)\,ds = O_k(\tau^2),
\]
and the cycle-averaging residual contributes
\[
    \int_0^\tau O_k(\delta^2)\,ds = O_k(\delta^2\tau).
\]
For short fixed post-branch windows, we absorb this into the stated residual notation as $O_k(\delta^2)$. Therefore
\[
    \bar\Delta\ell_k(\tau)
    =
    (1-c)R_k\tau
    +
    \frac{\alpha}{\beta}Q_k\tau
    +
    O_k(\tau^2)
    +
    O_k(\delta^2).
\]
This proves the decomposition.
\end{proof}

Proposition~\ref{prop:decomposition} isolates two distinct mechanisms in the branching intervention. The first is an ordinary learning-rate confounder: lowering $\eta$ by factor $c$ slows progress on every subset by $(1-c)R_k\tau$. The second is the EoS-specific selector: the baseline branch carries an additional $-(\alpha/\beta)\nabla S$ drift, contributing $(\alpha/\beta)Q_k\tau$ to the differential. Whether subset $k$ benefits from EoS depends on the sign of $Q_k$ relative to the rate confounder.

\begin{corollary}[Rate slowdown alone cannot produce mixed signs]
\label{cor:mixed-signs}
If $R_k \geq 0$ for all prototype groups $k$, then the learning-rate confounder $(1-c)R_k\tau$ is nonnegative for all $k$. A mixed-sign pattern across groups in $\bar\Delta\ell_k$ therefore cannot be explained by ordinary learning-rate slowdown alone.
\end{corollary}

\begin{proof}
$0 < c < 1$ implies $1 - c > 0$. With $\tau \geq 0$, the rate term has the sign of $R_k$. Nonnegativity of all $R_k$ then forces nonnegative rate contributions for all subsets.
\end{proof}

This is the theoretical basis for interpreting the selective trade-off in Figure~\ref{fig:eos_tradeoff}. Mixed-sign $\bar\Delta\ell_k$ across groups requires either the EoS selector $Q_k$ or another subset-dependent mechanism, not rate slowdown alone. In our empirical comparisons, we instead align branches by learning-rate-normalized time. This removes the leading-order speed difference along $-\nabla L$, so that the remaining first-order branch differential is governed by the EoS-specific selector $(\alpha/\beta)Q_k$. Higher-order differences due to the changed trajectory and sharpness threshold are absorbed into the residual terms.

\subsection{Curvature influence as a single-mode proxy}
\label{app:single-mode}

The main text measures
\[
    C_k := (\nabla\ell_k \cdot v_1)^2.
\]
Proposition~\ref{prop:decomposition} states the exact local selector as $Q_k = \langle\nabla\ell_k, \nabla S\rangle$. Under standard eigenvalue perturbation, $\nabla S = \nabla^3 L[v_1, v_1]$, so $\nabla S$ has a component along $\vone$ of magnitude $\gamma := \nabla^3 L[v_1, v_1, v_1]$ plus an orthogonal residual $r_\perp$. This splits the selector as
\[
Q_k \;=\; \underbrace{\gamma\langle \nabla\ell_k, v_1\rangle}_{T_1^{(k)}}
\;+\; \underbrace{\langle \nabla\ell_k, r_\perp\rangle}_{T_2^{(k)}},
\]
and $C_k$ functions as a single-mode proxy for the squared selector magnitude when the $v_1$-component term $T_1^{(k)}$ determines the sign and ranking of $Q_k$. Appendix~\ref{app:qk} tests this condition by evaluating $Q_k$ directly. In what follows, $C_k$ is the measured statistic. The factorization
\[
    C_k = \|\nabla\ell_k\|^2 \cos^2\theta_k,
    \qquad
    \cos^2\theta_k := \frac{(\nabla\ell_k \cdot v_1)^2}{\|\nabla\ell_k\|^2},
\]
separates direction (alignment) from magnitude (persistence). Appendices~\ref{app:coherence-lemma} and~\ref{app:saturation-lemma} establish that both factors are necessary.

\subsection{Single-Mode Approximation at EoS Onset}
\label{app:single}

We now verify the single-mode assumption empirically. Tracking the top three full-batch Hessian eigenvalues confirms that at EoS onset only $\lambda_1$ is near the stability boundary: $\lambda_1 = 200.3$ (within $0.2\%$ of $2/\eta$), $\lambda_2 = 99.7$, and $\lambda_3 = 46.0$, giving $\lambda_2/\lambda_1 = 0.498$. The gap narrows later in training ($\lambda_2/\lambda_1$ reaches $0.93$), yet even at $\lambda_2/\lambda_1 = 0.90$ the full-batch gradient places $95\%$ of its energy in $\vone$ versus $3.6\%$ in $\mathbf{v}_2$, so a narrowing eigengap does not by itself produce a two-dimensional update. The low-dimensional dominant spectrum is consistent with prior spectral analyses \citep{sagun2017hessian}.

Table~\ref{tab:single_mode} reports per-group coupling to the top two modes in the early \eos{} window following onset ($t_{\rm eff}\in[6,15]$). The outlier groups driving the instability couple to $\vone$ by one to two orders of magnitude over $\mathbf{v}_2$, so the single-mode analysis captures the groups on which the mechanism operates. The approximation is therefore sufficient for the onset-centered selection claim, though not for late training, where the spectrum and couplings shift (Figure~\ref{fig:single_mode}).

\begin{table}[h]
\centering
\caption{Per-group coupling to the top two Hessian modes, medians over checkpoints in the early \eos{} window following onset ($t_{\rm eff}\in[6,15]$).}
\label{tab:single_mode}
\begin{tabular}{lcc}
\toprule
Group & $\cos^2(\nabla\ell_k, \vone)$ & $\cos^2(\nabla\ell_k, \mathbf{v}_2)$ \\
\midrule
Input-outliers  & 0.68 & 0.003 \\
Output-outliers & 0.47 & 0.04 \\
Boundary points & 0.09 & 0.12 \\
Inliers         & 0.02 & 0.16 \\
\bottomrule
\end{tabular}
\end{table}

\begin{figure}[h]
  \centering
  \includegraphics[width=\linewidth]{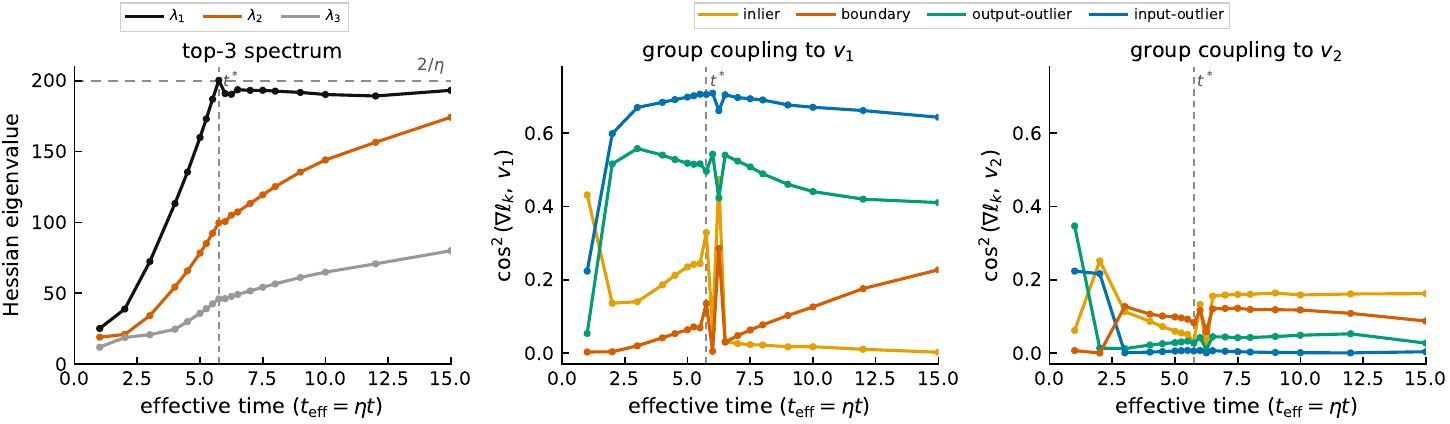}
  \caption{Single-mode structure at \eos{} onset. \textbf{Left:} top-3 full-batch Hessian eigenvalues against the stability boundary $2/\eta$. \textbf{Middle:} per-group coupling $\cos^2(\nabla\ell_k, \vone)$. \textbf{Right:} the same quantity on $\mathbf{v}_2$. $t^*$ is marked on all panels. Single representative run.}
  \label{fig:single_mode}
\end{figure}

\subsection{Direct validation of the single-mode proxy}
\label{app:qk}

We evaluate the selector $Q_k = \langle \nabla\ell_k, \nabla S \rangle$ directly. For a simple top eigenvalue, the Hellmann--Feynman identity gives $\nabla_\theta \lambda_1 = \nabla_\theta\, v_1^\top H(\theta)\, v_1$ with the normalized eigenvector held fixed, which we evaluate by higher-order automatic differentiation of the top-eigenvalue Rayleigh quotient at the converged LOBPCG eigenvector. This removes finite-difference approximation error. The remaining error comes from eigenvector convergence and floating-point arithmetic, and we control it by verifying a normalized eigenpair residual below $10^{-4}$ at every checkpoint. As a numerical precaution we restrict to checkpoints with relative eigengap $(\lambda_1-\lambda_2)/\lambda_1 \geq 0.05$, retaining 84--85 of 97 checkpoints per seed. The cutoff is a chosen criterion rather than a requirement of the identity, and the results do not depend on it (pooled Spearman $0.80$ ungated vs.\ $0.81$ gated).

We evaluate the proxy through three complementary comparisons. The first comparison tests magnitude. Figure~\ref{fig:scatter} plots $Q_k^2$ against the nonnegative $C_k$ across (subgroup, checkpoint) pairs, showing a monotone positive relationship spanning several orders of magnitude, with the dashed line marking the median ratio. The second comparison tests subgroup ranking. Comparing the normalized signed selector $\cos(\nabla\ell_k, \nabla S) = Q_k/(\|\nabla\ell_k\|\,\|\nabla S\|)$ with the normalized proxy $\cos^2(\nabla\ell_k, v_1)$ across three seeds gives pooled Spearman $\rho = 0.81$, median per-checkpoint Kendall $\tau = 0.67$ across the four groups, and the same dominant subgroup at $77\%$ of checkpoints ($25\%$ under uniform random selection). The dominant-subgroup agreement is unchanged ($77\%$) when computed on the unnormalized $Q_k$ and $C_k$. Because $C_k$ is nonnegative, these two comparisons test magnitude and ranking, not the selector's sign.

The third comparison tests the sign. Although $\nabla S$ is only weakly aligned with $v_1$ (median $\cos^2(\nabla S, v_1) = 0.05$), Table~\ref{tab:qk_decomp} shows that the top-mode term $T_1^{(k)}$ closely tracks the signed selector's ranking for the non-inlier groups ($\rho(Q_k, T_1^{(k)}) = 0.91$--$0.96$), with sign agreement strongest for output- and input-outliers ($0.97$ and $0.99$) and weaker for boundary points ($0.76$). For input-outliers the orthogonal residual is comparable in magnitude to the top-mode term, yet the top-mode term remains highly predictive of the selector's sign and ranking. For inliers the $v_1$-component is roughly an order of magnitude smaller than for the outlier groups and $C_k$ sits near its floor, so the proxy carries little information for this group, consistent with the selector predicting no \eos{} advantage there.

\begin{table}[h]
\centering
\caption{Decomposition of the directly evaluated selector, medians over three seeds. $T_1^{(k)} = \gamma\langle\nabla\ell_k, v_1\rangle$ is the $v_1$-component term, $T_2^{(k)} = \langle\nabla\ell_k, r_\perp\rangle$ the orthogonal residual term.}
\label{tab:qk_decomp}
\begin{tabular}{lccc}
\toprule
Subgroup & med $|T_2|/|T_1|$ & $\rho(Q_k, T_1^{(k)})$ & sign agreement \\
\midrule
Output-outliers & 0.31 & 0.96 & 0.97 \\
Input-outliers  & 1.05 & 0.96 & 0.99 \\
Boundary points & 0.53 & 0.91 & 0.76 \\
Inliers         & 3.54 & 0.78 & 0.81 \\
\bottomrule
\end{tabular}
\end{table}

\begin{figure} [h]
    \centering
    \includegraphics[width=0.6\linewidth]{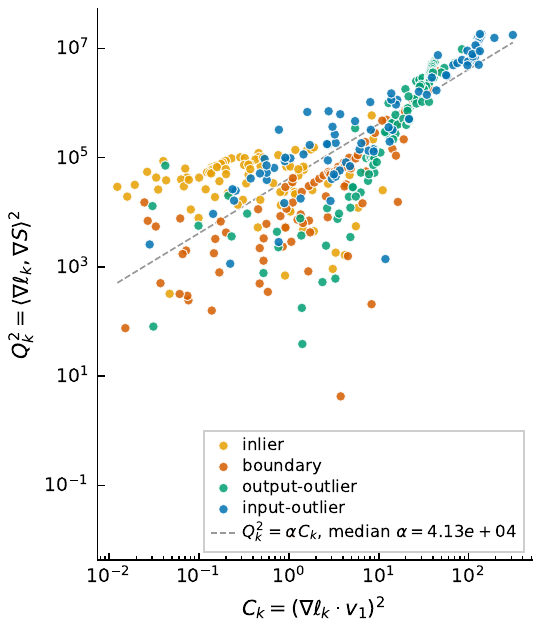}
    \caption{
        \textbf{Empirical validation of the single-mode proxy.}
        Scatter of $Q_k^2 = \langle \nabla \ell_k, \nabla S \rangle^2$, evaluated directly by higher-order automatic differentiation, versus $C_k = (\nabla \ell_k \cdot v_1)^2$ across (subgroup, checkpoint) pairs in training. The dashed line shows the median ratio $Q_k^2 / C_k = 4.1 \times 10^4$. The relationship is monotone for the outlier groups. For inliers, which are nearly orthogonal to $v_1$, $C_k$ is small and the proxy is loosest, in the regime where the selector predicts no \eos{} advantage.
    }
    \label{fig:scatter}
\end{figure}

\subsection{Alignment: coherent vs.\ random gradients}
\label{app:coherence-lemma}

The random-direction outlier ablation tests whether large geometric distance is sufficient for curvature dominance. The theory predicts that it is not: what matters is whether per-example gradients add coherently in a shared direction.

Let $g_k := \nabla\ell_k = m^{-1}\sum_{i=1}^m g_i$ be the group gradient, and let $u$ be a unit direction (interpreted as the current top Hessian eigenvector).

\begin{lemma}[Coherence amplifies curvature influence]
\label{lem:coherence}
Suppose first that per-example gradients share a coherent component:
\[
    g_i = a_i q + \varepsilon_i,
    \qquad
    m^{-1}\sum_i \langle \varepsilon_i, u\rangle \approx 0.
\]
Then $\langle g_k, u\rangle \approx \bar a\, \langle q, u\rangle$ with $\bar a := m^{-1}\sum_i a_i$.

Suppose instead that per-example directions are independent, mean-zero, and isotropic in an effective $d_{\text{eff}}$-dimensional subspace (i.e. incoherent directions):
\[
    g_i = a q_i,
    \qquad \mathbb{E}[q_i] = 0,
    \qquad \mathbb{E}\langle q_i, u\rangle^2 = 1/d_{\text{eff}}.
\]
Then
\[
    \mathbb{E}\!\left[\langle g_k, u\rangle^2\right] = \frac{a^2}{m\, d_{\text{eff}}}.
\]
\end{lemma}

\begin{proof}
In the coherent case,
\[
    \langle g_k, u\rangle = \bar a \langle q, u\rangle + m^{-1}\sum_i \langle \varepsilon_i, u\rangle,
\]
and the residual average is negligible by assumption. In the random case, the terms $\langle q_i, u\rangle$ are independent with mean zero, so
\[
    \mathbb{E}\!\left[\langle g_k, u\rangle^2\right]
    = \frac{a^2}{m^2}\sum_i \mathbb{E}\langle q_i, u\rangle^2
    = \frac{a^2}{m\, d_{\text{eff}}}.
\]
\end{proof}

Lemma~\ref{lem:coherence} justifies the coherent-vs-random ablation in Section~\ref{sec:directional_alignment} as a test of the alignment mechanism. If \eos{} selectivity is driven by directional alignment rather than distance or gradient magnitude alone, then preserving displacement size while randomizing directions should collapse the group-level projection onto $\vone$ and eliminate the \eos{} advantage. In the coherent case, the projection remains order $\bar a\langle q,\vone\rangle$, and hence $C_k$ remains order $\bar a^2\langle q,\vone\rangle^2$. In the random-direction case, the expected squared projection is only $a^2/(m d_{\mathrm{eff}})$, so curvature influence averages away with group size and effective dimension. A norm-only account predicts no difference between the two conditions; the observed collapse of the \eos{} advantage under random-direction displacement is therefore evidence for alignment as the operative mechanism, not magnitude alone.

\subsection{Persistence: gradient saturation removes curvature influence}
\label{app:saturation-lemma}

The MSE-vs-CE comparison tests whether alignment is sufficient when gradient magnitude collapses. Write $\ell_i(\theta) = \phi(f_\theta(x_i), y_i)$ and decompose
\[
    \nabla_\theta \ell_i = J_i^\top r_i,
    \qquad
    J_i := \nabla_\theta f_\theta(x_i),
    \qquad
    r_i := \nabla_f \phi(f_\theta(x_i), y_i).
\]

\begin{lemma}[Saturation removes curvature influence]
\label{lem:saturation}
Suppose $\|r_i\| \to 0$ for all $i \in P_k$ and the Jacobians are uniformly bounded, $\|J_i\|_{\text{op}} \leq B$. Then $\|\nabla\ell_k\| \to 0$, and consequently
\[
    (\nabla\ell_k \cdot u)^2 \to 0
\]
for every unit vector $u$, regardless of the alignment $\cos^2\theta_k$.
\end{lemma}

\begin{proof}
Per example, $\|\nabla_\theta \ell_i\| \leq \|J_i\|_{\text{op}} \|r_i\| \leq B \|r_i\|$. Hence
\[
    \|\nabla\ell_k\| \leq |P_k|^{-1}\sum_{i \in P_k} \|\nabla_\theta \ell_i\| \to 0,
\]
and $|\langle \nabla\ell_k, u\rangle| \leq \|\nabla\ell_k\|$ then gives the second claim.
\end{proof}

For softmax cross-entropy, $r_i = p_i - y_i$, so confidently correctly classified examples ($p_i \to y_i$) drive $r_i \to 0$ and the corresponding subset gradient norm collapses. Output-outliers, whose assigned labels are inconsistent with their input, retain $\|p_i - y_i\|$ bounded away from zero as long as the model continues to predict the input-consistent class, so their gradients persist. Lemma~\ref{lem:saturation} thus predicts that under CE, a confidently classified subset can retain high $\cos^2\theta_k$ while losing all curvature influence, and that the EoS advantage transfers to whichever subset retains non-vanishing gradients---empirically, output-outliers.

The MSE comparison sharpens the test. MSE residuals also vanish on perfectly fit examples, but in the observed regime the coherent input-outliers retain large residuals throughout training, preserving gradient magnitude. The MSE-vs-CE contrast therefore isolates persistence: a subset that remains aligned with $\vone$ but loses gradient norm loses EoS influence, while a subset that retains both keeps it.

\subsection{Scope of the results}
\label{app:scope}

The results above are local and conditional. Under the Damian--Nichani--Lee self-stabilization approximation (A2), Proposition~\ref{prop:decomposition} establishes that the cycle-averaged branch differential decomposes additively into a learning-rate confounder controlled by $R_k$ and an EoS-specific selector controlled by $Q_k$. Lemmas~\ref{lem:coherence} and~\ref{lem:saturation} explain why directional coherence and gradient persistence are each individually necessary for a subset to feel the selector. The results do not claim that the decomposition holds globally, that the single-mode proxy $C_k$ exactly equals $|Q_k|^2$, or that flatness has a universal functional meaning. Empirically, the directly evaluated selector tracks the proxy in both magnitude and ranking across (subgroup, checkpoint) pairs in the \eos{} regime (Appendix~\ref{app:qk} and Figure~\ref{fig:scatter}), validating its use as a single-mode proxy. The functional consequence of \eos{} depends on which subset dominates the selector at training time---a point developed empirically in Section~\ref{sec:mechanism}.



\end{document}